\documentclass[journal, onecolumn, 10pt, letterpaper]{IEEEtran}
\usepackage[letterpaper, top=2cm, bottom=2.1cm, left=2cm, right=2cm]{geometry}

\usepackage{graphicx} 
\usepackage{amsmath}
\usepackage{amssymb}
\usepackage[normalem]{ulem}
\usepackage{algorithm}
\usepackage{xcolor}
\usepackage{bbm}
\usepackage{mathtools}
\usepackage{hyperref}
\usepackage{mathrsfs}
\usepackage{algpseudocode}
\usepackage{systeme}
\usepackage{changepage}
\usepackage{amsthm}
\usepackage{thmtools}
\usepackage{thm-restate}
\usepackage{enumitem}
\usepackage{cleveref}
\usepackage[normalem]{ulem}

\usepackage{mathtools}

\usepackage{tikz}
\usetikzlibrary{fit}
\usetikzlibrary{decorations.pathreplacing}

\DeclarePairedDelimiter\floor{\lfloor}{\rfloor}

\def \P {{\mathcal{P}}}
\def \Ev {{\mathcal{E}}}

\def \I {{\mathbbm{1}}}
\def \vY {{\mathbf{Y}}}

\def \vI {{\mathbf{I}}}
\def \vJ {{\mathbf{J}}}

\newcommand{\Alt}{\mathrm{Alt}}
\newcommand{\KL}{\mathrm{KL}}
\newcommand{\Sel}{\textsc{Select}}
\newcommand{\Alg}{\mathrm{A}^3\mathrm{CNP}}

\mathtoolsset{showonlyrefs}

\newcommand{\Ber}{\mathrm{Ber}}

\usepackage{preamble_IEEE}

\begin{document}

\title{Almost Asymptotically Optimal Active   Clustering
Through Pairwise Observations} 


\author{%
  \IEEEauthorblockN{Rachel~S.~Y.~Teo\IEEEauthorrefmark{1} ~~ P.~N.~Karthik\IEEEauthorrefmark{2} ~~ Ramya~Korlakai~Vinayak\IEEEauthorrefmark{3} ~~ Vincent~Y.~F.~Tan\IEEEauthorrefmark{1}} \\
  \IEEEauthorblockA{\IEEEauthorrefmark{1}National University of Singapore ~~ \IEEEauthorrefmark{2} Indian Institute of Technology Hyderabad ~~  \\ \IEEEauthorrefmark{3} University of Wisconsin-Madison \\
Emails: \quad rachel.tsy@u.nus.edu ~~ pnkarthik@ai.iith.ac.in ~~ ramya@ece.wisc.edu ~~ vtan@nus.edu.sg}
}

\maketitle

\begin{abstract}
We propose a new analysis framework for  clustering $M$ items into an unknown number of $K$ distinct groups using noisy and actively collected   responses. At each time step, an agent is allowed to query pairs of items and observe bandit binary feedback. If the pair of items belongs to the same (resp.\ different) cluster, the observed feedback is $1$ with probability $p>1/2$ (resp.\ $q<1/2$). Leveraging the ubiquitous change-of-measure technique, we establish a fundamental lower bound on the expected number of queries needed to achieve a desired confidence in the clustering accuracy, formulated as a sup-inf optimization problem. Building on this theoretical foundation, we design an asymptotically optimal algorithm in which the stopping criterion involves an empirical version of the inner infimum---the Generalized Likelihood Ratio (GLR) statistic---being compared to a threshold. We develop a computationally feasible variant of the GLR statistic and show that its performance gap to the lower bound can be accurately empirically estimated and  remains within a constant multiple of the lower bound.
\end{abstract} 

\section{Introduction}

Clustering is a fundamental task in unsupervised learning, and is concerned with partitioning a set of items into groups such that items within the same group are more similar to one another than to items in different groups \cite{driver1932quantitative,RUSPINI196922,jain1999data,xu2005survey}. Classical approaches such as $K$-means \cite{macqueen1967kmeans}, hierarchical clustering \cite{johnson1967hierarchical}, and spectral methods \cite{ng2002spectral} typically rely on explicit similarity measures to uncover cluster structure. An alternative model---clustering with the help of a noisy oracle---proposed by \cite{mazumdar2017clustering}, uses crowdsourcing to uncover underlying structure in the data. 

In this model, we are given a set of $M$ unlabeled items to be clustered into $K$ clusters, where $K$ is unknown. There is an oracle which, when queried with two items, specifies if they belong to the same cluster. However, the answer could be wrong with a certain probability. The motivation for this model lies in its practical relevance, for example to crowdsourced entity resolution  \cite{fellegi1969theory, elmagarmid2007duplicate, getoor2012entity}, signed edge prediction in social networks \cite{leskovec2010predicting, burke2008mopping, brzozowski2008friends, chen2016edge} and correlation clustering \cite{bansal2004correlation,ailon2018approximate, saha2019correlation}. 

Clustering with noisy queries is closely related to the Stochastic Block Model (SBM), a widely recognized framework for analyzing graph clustering and community detection problems \cite{holland1983stochastic, condon2001graph, abbeSandon2015community, abbe2018community}. In the SBM, a set of $M$ vertices (items) is divided into $K$ hidden partitions, and random graphs are generated based on these vertex memberships. An edge is formed between any two vertices with a probability $p$ if they belong to the same cluster, and with probability $q$ if they do not, where $q<p$. The goal is to recover the underlying clusters from a random graph sampled according to the SBM.

However, our framework differs in several key aspects from SBM studies. First, we obtain {\em multiple adaptively sampled} queries for each pair of items, while the SBM provides a {\em single snapshot} of the graph's connectivity that reflects the relationships among the items. Second, akin to fixed-confidence best arm identification \cite{garivier2016optimal}, we investigate the asymptotic regime where the probability of error in recovering the true clustering approaches zero. This regime is natural in high-assurance entity resolution tasks such as patient record linkage across hospital information systems: here, $M$ denotes the (finite) set of patient records, pairwise queries correspond to costly manual or expert checks, and an erroneous {\em merge} (conflating two patients) or {\em split} (fragmenting one patient into multiple identities) can have severe downstream clinical consequences. Consequently, system requirements are often specified as a stringent target error probability level, making the asymptotics of vanishing error probabilities more relevant than the traditional limit of letting $M \to \infty$ in active clustering. Lastly, the probabilities $p=p_M$ and $q=q_M$ in SBMs are functions of $M$, typically scaling as $\Theta\big(\frac{\log M}{M}\big)$. In our context, however, $p$ and $q$ are treated as fixed but unknown constants.


{\bf Contributions.} We introduce a theoretically grounded algorithm that minimizes the expected number of queries in the regime of vanishing error probabilities. 
This setting aligns closely with \cite{ChenVinayakHassibi2023}. Our main contributions are:
\begin{itemize}[leftmargin=*]
    \item We formulate the clustering with noisy oracle problem as one with bandit feedback, with an arm representing a pair of items to query. We define an expression for an instance of the problem, characterized by the unique partition of the $M$ items and the  probabilities of the oracle's feedbac  $0 \le q < 1/2< p \le 1$.  
    \item We derive an asymptotic, instance-dependent lower bound on the expected sample complexity, expressed as a weighted binary KL divergence sup-inf optimization  that captures the hardness of the instance relative to its closest alternatives. Leveraging the structure of the problem, we significantly reduce the search space of alternatives, enabling the design of a principled sampling rule.
    \item We develop an asymptotically optimal algorithm that, although statistically optimal, relies on a computationally intractable stopping rule. In response, we propose a practical, computationally feasible variant whose performance, while suboptimal, admits  a multiplicative upper bound on its deviation from optimality (the asymptotic lower bound) as well as an empirical approximation that converges to the theoretical upper bound.
\end{itemize}

{\bf Related Works.} Mazumdar and Saha \cite{mazumdar2017clustering} pioneered the study of clustering with noisy queries, introducing a computationally intractable algorithm that approaches their information-theoretic lower bound on the sample complexity. They also proposed a more efficient but sub-optimal algorithm with quadratic growth in sample complexity relative to the number of clusters. Building on this, \cite{mukherjee2023recovering} proposed algorithms to recover clusters of size $\Omega(K \log M)$ with fixed error probability, while \cite{10.1145/3366423.3380045} enhanced methods for the two-cluster case and \cite{peng2021query} reduced dependencies on $p-q$. Recently, \cite{xia2022optimal} established a new lower bound and presented a polynomial-time algorithm achieving an upper bound up to an additive term by linking to multi-armed bandits, albeit using a different approach of forming sub-clusters. Our work addresses a practical scenario with unknown $p$ and $q$ and allows repeated queries. Relatedly, \cite{ChenVinayakHassibi2023} proposed a computationally efficient active clustering algorithm based on the law of iterated logarithms (LIL) for multi-arm bandits, ensuring high-probability recovery with nearly optimal sample complexity.  Previous efforts using passive strategies have limited effectiveness with small clusters, focusing instead on all item pairs or randomly or carefully selected subsets of pairs \cite{NIPS2014_9a2cb485, vinayak2016crowdsourced, gomes2011crowdclustering, 9529053}.

Our approach builds on techniques from \emph{fixed-confidence pure exploration}, particularly those achieving asymptotic optimality as error probability vanishes~\cite{garivier2016optimal, jedra2020optimal}. Similar to~\cite{yang2024}, we frame our online clustering task as a pure exploration problem using pairwise comparisons, akin to dueling bandits~\cite{yue2012}.  While some studies focus on clustering structures in multi-armed bandit algorithms for regret minimization \cite{nguyen2014dynamic, gentile2014online, li2016collaborative, bouneffouf2019optimal, singh2020multi}, our emphasis lies in uncovering the actual partition of the arms based on pairwise feedback.

\section{Problem Setup and Preliminaries}
\label{sec:setup}
Consider $M$ items, each of which belongs to one of $K \in [M]$ clusters, where $[M] \coloneqq \{1, \ldots, M\}$. Their cluster assignments are specified by $c_i \in [K]$ for $i \in [M]$. For $\ell \in [K]$, we define $g_\ell \coloneqq \{i \in [M]: c_i=\ell\}$ as the set of items in cluster $\ell$. We assume that the number of clusters $K$ is not known beforehand. Let $I = \{(i,j) \in [M] \times [M]: i < j \}$ be an index set that represents all  possible distinct pairs of items. We denote by $\P_{I}$ the set of distributions on~$I$. 

At each time  $t \in \mathbb{N}$, we query whether items $\vI_t = i_t$ and $\vJ_t = j_t$ belong to the same cluster and obtain an observation $\vY_{i_t j_t}$ given by
\begin{align}
    \vY_{i_t j_t}  \sim v_{i_t j_t} = \begin{cases}
        \Ber(p), & \text{if $c_{i_t}=c_{j_t}$}, \\
        \Ber(q), & \text{if $c_{i_t}\ne c_{j_t}$},
    \end{cases}
    \label{eq:feedback}
\end{align}
for $0 \le q < 1/2< p \le 1$. That is, the observation  $\vY_{i_t j_t} $ obtained upon querying any pair of items follows a Bernoulli distribution with parameter $p$ (resp.~$q$) if the items are in the same cluster (resp.~distinct clusters). 
The parameters $p$ and $q$ are assumed to be unknown. 
We let $N_{ij}(t)$ denote the number of feedback observations from queries about 
 the pair $(i,j) \in I$  up to time $t$. 

Let $\mathcal{S}^M \coloneqq \{ C \in \R^{M \times M}: C^\top = C\} $ be the set of symmetric matrices of size $M$ by $M$ and $\mathcal{C} = \big\{  C\in \mathcal{S}^M  : \forall (i,j) \in I, c_{ij} = \I \{c_i = c_j\} p + \I \{c_i \ne c_j\} q  \;\mbox{ for some }\; 0\le q < 1/2 < p\le 1 \big\}$. The set $\mathcal{C}$ characterizes all {\em problem instances}, and consists of all possible clustering of the $M$ items combined with the different values of $p$ and $q$. 

Our objective  for the above {\em crowdsourced clustering problem}  is to find the correct clustering of the items based on the observations.    We study the {\em   fixed-confidence} setting where, given a confidence level $\delta \in (0,1)$, we wish to design an algorithm that (almost) minimizes the expected number of time steps, denoted as $\E [\tau_\delta]$, required to identify the correct clustering with probability  $\ge 1-\delta$ (a {\em $\delta$-correct algorithm}). In particular, we seek to design and analyze algorithms that are close to being asymptotically optimal, in a sense to be made precise. 

{\bf Equivalence Relation:} Let $C_=$ and $C_{\ne}$ be $\{0,1\}$-valued matrices such that $(c_=)_{ij} = \I \{c_i = c_j\}$ and $(c_{\ne})_{ij} = \I \{c_i \ne c_j\}$. Define the relation $\sim$ on $\mathcal{C}$ as follows: For $C, C' \in \mathcal{C}$, let $C \sim C'$ if $C_= = C'_=$ and $C_{\ne} = C'_{\ne}$. It is easy to verify that $\sim$ is an equivalence relation. This partitions the set of all possible instances as $\mathcal{C} = \bigcup_{K=1}^{M} \{C' \in \mathcal{C}: \exists\, C \in \mathcal{C} \text{ with } K \text{ clusters and } C' \sim C\} := \bigcup_{K \le M} \mathcal{C}^K$, where $\mathcal{C}^K$ represents the set of all instances having exactly $K$ clusters. 
We say that $C$ has $K$ clusters if $C$ represents cluster assignments such that there are exactly $K$ different clusters. Each partition $\mathcal{C}^K$ is characterized by the number of clusters $K$ that each instance represents. 
Let $\mathcal{C}/\sim$ denote the set of all equivalence classes, and let $[C] \in  \mathcal{C}/\sim$ denote the equivalence class associated with instance~$C$. Let $f: \widetilde{\mathcal{C}} \rightarrow \mathscr{P}\big(\mathscr{P}(M)\big)$ be defined as $f([C])=\{g_1,g_2,\ldots,g_K\}$ for $C \in \mathcal{C}^K$, where $g_1,g_2,\ldots,g_K$ represent the cluster assignments in $C$, and $\mathscr{P}(M)$ denotes the power set of $[M]$.  Observe that a cluster assignment of the form $\{g_1, \ldots, g_K\}$ is an element of $\mathscr{P}\big(\mathscr{P}(M)\big)$.  For ease of understanding, we provide an example in Appendix \ref{app:eg}. 

{\bf Notation:} Let $\mathcal{F}_t = \sigma ( \{(\vI_s,\vJ_s) , \vY_{\vI_s\vJ_s} \}_{s=1}^t  )$  be the $\sigma$-algebra generated by the observations up to time~$t$. 
Let $\tau_\delta$ be a stopping time with respect to the filtration $\{ \mathcal{F}_t \}_{t \in \mathbb{N}}$.

\section{Lower Bound}
\label{sec:lower_bound}
We establish an instance-dependent lower bound on the expected sample complexity $\E_C [ \tau_\delta]$ in the active crowdsourced clustering problem. For an instance $C\in \mathcal{C}$, let
\begin{align*}
\Alt(C) \coloneqq \bigcup_{\substack{[C']\in \widetilde{\mathcal{C}} :~ 
[C'] \ne [C]}} [C']
\end{align*} 
denote the set of {\em alternative} instances that represent an incorrect clustering different from that of $C$. 
\begin{restatable}{theorem}{lowerbound}
\label{thm:lower_bound}
    For a   confidence level $\delta \in (0,1)$ and instance $C\in \mathcal{C}$, any $\delta$-correct algorithm 
    satisfies 
    \begin{align*}
       \E_C [ \tau_\delta] \ D^*(C) \, \ge \, {\rm KL} \big(\mathrm{Ber}(\delta) , \mathrm{Ber}(1-\delta) \big), 
    \end{align*}
    where 
    \begin{align}
    D^*(C)   \coloneqq \sup_{\lambda \in \P_{I}} \inf_{C' \in \Alt(C)}  \sum_{(i,j) \in I}  \lambda_{ij} \KL(v_{ij}, v'_{ij}),     \label{eqn:d_star_C}
    \end{align}
    and $v'_{ij}$ is the   distribution for $C'$ analogous to \eqref{eq:feedback} (with corresponding parameters $p' > 1/2 > q'$). Furthermore,
    \begin{align}
    \liminf_{\delta \rightarrow 0}\frac{\E_C[\tau_\delta]}{\ln(1/\delta)} \ge D^*(C)^{-1}. 
    \label{eq:lower-bound}
    \end{align}
\end{restatable}
The proof of Theorem \ref{thm:lower_bound} can be found in Appendix \ref{app:prf_lower_bd}.
\begin{remark}
\label{rmk:lower_bound}
    {\em We compute the  expected sample complexity in terms of $\Delta = p-q$ and $K$ on a simple example. Consider an instance in which the number of clusters $K=M/2$, the (even) number of items $M>4$, and each cluster has exactly two items, i.e., $|g_i|=2$ for all $i$. Evaluating $D^*(C)$ explicitly, we find that
    \begin{align}
      \liminf_{\delta \rightarrow 0}\frac{\E_C[\tau_\delta]}{\ln(1/\delta)}  \ge \underbrace{\biggl[  \Delta \log\frac{1+\Delta}{1-\Delta} \biggr]^{-1}}_{\Theta(1/\Delta^2)  \text{ as }  \Delta\to0 } \underbrace{\frac{K^4-5K^2+4K}{K^3-K^2-2}}_{\Theta(K)   \text{ as }  K \to\infty}.  \label{eqn:explicit_lb}
    \end{align}  
    We highlight that the framework used to derive our  bound differs from those employed in the literature. Here,  the expected number of pairwise queries $\E_C[\tau_\delta]$ is normalized by $\ln(1/\delta)$, and $\delta\to 0$ in~\eqref{eqn:explicit_lb}. This approach is akin to  work on fixed-confidence best arm identification \cite{garivier2016optimal}.  In contrast, other studies such as \cite{mazumdar2017clustering} and \cite{ChenVinayakHassibi2023} derive high-probability bounds for $\tau_\delta$ without normalization, and their results typically scale as $\tau_\delta=\widetilde{O}(MK/\Delta^2)$. In these analyses, the failure probability is also chosen to decay polynomially in $(M,K)$. Consequently, while the dependence of our bound in~\eqref{eqn:explicit_lb} on $\Delta$ aligns with that of previous works, the dependence on $K$ appears to be less consistent. 
     as in our current example, $MK = \Theta(K^2)$. 
Nonetheless, direct comparison is not possible due to  the normalization  $\ln(1/\delta)$ and the  limit as $\delta\to0$.}
\end{remark}

The quantity $D^*(C)$ in \eqref{eqn:d_star_C} captures the ``hardness'' to identify the clustering in $C$ as a function of the KL divergences between the Bernoulli distributions corresponding to $C$ (see~\eqref{eq:feedback}) and that of its most confusing alternative $C'$. By definition, $D^*(C)$ is the solution to a sup-inf   problem, and $D^*(C)^{-1}$ is the instance-dependent lower bound. 
Solving for $D^*(C)$ yields $\Lambda^*(C)$, the set of all optimal allocations  across  pairs. Sampling according to any $\lambda^* \in \Lambda^*(C)$ at each time step provides the most informative choices of pairs to query to achieve the sample complexity lower bound in \eqref{eq:lower-bound}, effectively demonstrating its tightness as $\delta \to 0$.
We will use $\lambda^*$ to guide the sampling rule in our algorithm in Section \ref{sec:achievability} to provide the most informative choices of pairs to query to achieve the sample complexity lower bound in \eqref{eq:lower-bound}.

To do so, we need to solve for the infimum over $\Alt(C)$ appearing in the expression for $D^*(C)$ in \eqref{eqn:d_star_C}. However, this is  computationally expensive, because there are a substantial number of different equivalence classes $[C']$ in $\Alt(C)$, with each equivalence class governed by a unique clustering of the $M$ items.  The total number of equivalence classes, also known as the Bell number \cite{moser1955asymptotic}, grows super-exponentially in $M$. 
Additionally, within each equivalence class, there exist uncountably many alternative instances $C'$ parametrized by $(p',q')$ satisfying $0\le q'< 1/2< p' \le 1$. This leads to practical challenges in solving the infimum over $\Alt(C)$ and obtaining a closed-form expression for $\Lambda^*(C)$. We derive a simplified expression for the objective  in $D^*(C)$, and show that the infimum over $\Alt(C)$ is equal to the infimum over a   smaller set  $\min(C)$. Using these insights, we solve for $\Lambda^*(C)$. 

Below, we describe our simplification technique in more detail. For a fixed $C$ and $C' \in \Alt(C)$, let 
\begin{alignat*}{2}
    N_p &\coloneqq \{ (i,j) \in I: c_i = c_j \} \subset I, \\
    N_q &\coloneqq \{ (i,j) \in I: c_i \ne c_j \} \subset I, \\
    N_1 &\coloneqq \{ (i,j) \in I: c_i = c_j \text{ and } c'_i \ne c'_j \} \subset N_p, \\
    N_2 &\coloneqq \{ (i,j) \in I: c_i \ne c_j  \text{ and } c'_i = c'_j \} \subset N_q.
\end{alignat*}
Here, $N_p$ (resp.\ $N_q$) denotes the pairs of items belonging to the same cluster (resp.\ different clusters) in $C$, while $N_1$ (resp.\ $N_2$) denotes the pairs of items that belong to the same cluster (resp.\ different clusters) in $C$ but different clusters (resp.\ same cluster) in $C'$. For ease of notation, we also introduce mappings $\Lambda_{N_1}, \Lambda_{N_2}: \Alt(C) \rightarrow \mathscr{P}(I)$ such that $\lambda_{N_1}(C') = N_1$ and $\lambda_{N_2}(C') = N_2$. 
Notice that 
\begin{alignat*}{2}
    \KL(v_{ij},v'_{ij})&=d(p,q') \quad \forall (i,j) \in N_1, \\
    \KL(v_{ij},v'_{ij})&=d(q,p') \quad \forall (i,j) \in N_2, \\
    \KL(v_{ij},v'_{ij})&=d(p,p') \quad \forall (i,j) \in N_p\setminus N_1, \\ \KL(v_{ij},v'_{ij})&=d(q,q') \quad \forall (i,j) \in N_q \setminus N_2,
\end{alignat*}
where $p'$ and $q'$ are the Bernoulli parameters of $C'$. Using the above notations, we have
\begin{align}
\label{eqn:sup-inf}
    D^*(C) = \sup_{\lambda \in \P_{I}} \inf_{C' \in \Alt(C)} h(\lambda, C'),
    \end{align} 
    where $h(\lambda, C')$ is defined as the sum  
    \begin{align}
        h(\lambda, C')&=d(p,q')  \sum_{(i,j)\in \Lambda_{N_1}(C') } \lambda_{ij}+ d(q,q') \sum_{(i,j)\in N_q \setminus \Lambda_{N_2}(C') }\lambda_{ij} \nonumber\\* 
       &\quad+  d(q,p')   \sum_{(i,j)\in \Lambda_{N_2}(C')}\lambda_{ij} + d(p,p')\sum_{(i,j)\in N_p \setminus \Lambda_{N_1}(C') } \lambda_{ij}.
    \end{align}



\begin{restatable}{theorem}{innerinf}
\label{thm:inner_inf}  
Fix $M>K>1$. For $C \in \mathcal{C}^K$, let $\min(C)  \coloneqq \big(\mathcal{C}^{K-1} \cup \mathcal{C}^{K+1}\big) \cap \Alt(C)_{1}$, 
where $\Alt(C)_1$ is the set of alternative instances whose clustering may be obtained from that of $C$ by either splitting a cluster in $C$ into two clusters, or merging two clusters in $C$ into one.
Then,  we can shrink the $\Alt(C)$ set in~\eqref{eqn:sup-inf} to a considerably smaller subset $\min(C)$
\begin{align}
    D^*(C) = \sup_{\lambda \in \P_{I}}    \inf_{C' \in \Alt(C)} h(\lambda, C') 
    = \sup_{\lambda \in \P_{I}}  \inf_{C' \in \min(C)} h(\lambda, C').  \label{eq:inner_inf_minC}
\end{align}
\end{restatable}
The proof of Theorem \ref{thm:inner_inf} can  be found in Appendix \ref{app:proof_inner_inf}.
 
Theorem~\ref{thm:inner_inf} shows that if $C$ has $K$ clusters, the closest alternative instance to $C$ must have either $K+1$ or $K-1$ clusters in it, resulting from merging two clusters in $C$ into one or splitting a cluster in $C$ into two. Fixing $\lambda \in \P_{I}$, one may solve for the optimal $(p',q')$ within each equivalence class of $\min(C)$ in terms of $\lambda$ and the parameters $(p,q)$ governing~$C$. Taking a minimum over the finite number of equivalence classes of $\min(C)$ then yields $D^*(C)$.

Despite the above simplification, no closed-form solution for $\Lambda^*(C)$ exists in general, except in the special case when $K=1$ (all items belong to a single cluster) or $K=M$ (each cluster consists of a single item). In these special cases, the optimal allocation is unique (i.e., $\Lambda^*(C)$ is a singleton set) and equal to the uniform distribution over $I$,  as noted in the below result.
\begin{proposition}
\label{prop:uniform}
   For $C \in \mathcal{C}^K$, $K=1$ or $K=M$, let
\begin{align*}
\Lambda^*(C) = \argmax_{\lambda \in \P_{I}} \inf_{C' \in \Alt(C)} h(\lambda, C').
\end{align*} 
Then, $\Lambda^*(C)$ is the  uniform distribution over $I$. 
\end{proposition}
The proof of Proposition \ref{prop:uniform} can be found in Appendix~\ref{app:uniform}.

From Proposition \ref{prop:uniform}, it follows that when all items belong to a single cluster or when each item forms its own cluster, the optimal sampling rule should follow a uniform sampling across all pairs of items. 
This is because in this scenario,  each pair $(i,j) \in I$ yields Bernoulli observations with parameter $p$ (resp.\ $q$), thus rendering every pair equally informative. 

\section{Achieving the Lower Bound Up to a Constant}
\label{sec:achievability}

The lower bound in~\eqref{eq:lower-bound} suggests that any algorithm that samples items according to a distribution $\lambda^* \in \Lambda^*(C)$ at each time step achieves  the lower bound  as $\delta \to 0$. However, because the underlying instance $C$ is not known beforehand, it must be estimated along the way. The key issue here, as we note below, is that the estimated instance may not necessarily belong to the set $\mathcal{C}$ of all feasible instances. We carefully design an algorithm that, at each time step, projects the estimated instance onto $\mathcal{C}$, uses the projection to sample a pair of items, and proceeds until it is sufficiently confident of the underlying clustering. We improvise this algorithm to address the computational bottlenecks, and design a computationally efficient variant that is almost asymptotically optimal.
\subsection{Modified D-Tracking-based Sampling Rule}
\label{subsec:sampling-rule}

\begin{algorithm}
\caption{\quad  \textcolor{black}{$\mathrm{Proj}_{\mathcal{C}}(\hat{C}_t)$}}
\label{alg:proj}
\begin{algorithmic}[1]
\Require $\hat{C}_t$
\State Set $(\hat{c}_{t_=})_{ij} = \I \{(\hat{c}_t)_{ij} \ge 0.5 \}$ and  $(\hat{c}_{t_{\ne}})_{ij}=\I \{(\hat{c}_t)_{ij} < 0.5 \}$, for all $(i,j) \in I$
    \State Solve  
        \begin{align}
        \label{eq:proj_hatC}
            [C_t] = \argmin_{[C'] \in \widetilde{C}} \sum_{(i,j) \in I} |(\hat{c}_{t_=})_{ij}-(c'_=)_{ij}|+|(\hat{c}_{t_{\ne}})_{ij}-(c'_{\ne})_{ij}|.
         \end{align}
        \State Obtain $p_t$ and $q_t$ via
        \begin{align*}
        p_t &= \frac{\sum_{(i,j)\in I} \I \{(c_{t_=})_{ij}=1\} (\hat{c}_t)_{ij} \cdot N_{ij}(t)}{\sum_{(i,j)\in I} \I \{(c_{t_=})_{ij}=1\} N_{ij}(t)}, \quad
        q_t  = \frac{\sum_{(i,j)\in I} \I \{(c_{t_{\ne}})_{ij}=1\} (\hat{c}_t)_{ij} \cdot N_{ij}(t)}{\sum_{(i,j)\in I} \I \{(c_{t_{\ne}})_{ij}=1\} N_{ij}(t)}.
            \end{align*}
            \State Set $C_t = p_tC_{t_=} + q_t C_{t_{\ne}} \in \mathcal{C}$
        \State \textbf{Output:} $C_t$
\end{algorithmic}
\end{algorithm}

We now design a sampling rule that uses estimates of the unknown instance $C$ to suggest the most informative pair of items to query at each time step. 
Let $\widehat{C}_t$ be the empirical instance at time $t$ whose elements are given by 
\begin{align}
(\widehat{c}_t)_{ji}=(\widehat{c}_t)_{ij} =\frac{1}{  N_{ij}(t)}\sum_{s \le t} \vY_{i_s j_s} \I \{I_s = i, J_s = j\}.     
\end{align}
A na\"ive approach might be to use  $\widehat{C}_t$ as a proxy for $C$ to solve for $D^*(C)$ and obtain $\Lambda^*(C_t)$, and subsequently sample $(i,j) \sim \lambda^*$ for some $\lambda^* \in \Lambda^*(C_t)$. This approach fails if $\widehat{C}_t \notin \mathcal{C}$. Recall that any $C \in \mathcal{C}$ must necessarily satisfy $c_{ji}=c_{ki}$ for all $i,j,k \in [M]$ belonging to the same cluster. However, $(\widehat{c}_t)_{ji}$ and $(\widehat{c}_t)_{ki}$ may potentially be non-identical. 
 Hence, it is not appropriate to substitute  $\widehat{C}_t$ into~\eqref{eq:inner_inf_minC} directly. 
To address this, we construct a subroutine  $\mathrm{Proj}_{\mathcal{C}}$, detailed in Algorithm~\ref{alg:proj}, to find the element of $\mathcal{C}$ that is ``closest'' to $\widehat{C}_t$,   projecting $\widehat{C}_t$ onto $\mathcal{C}$. We  let $C_t \coloneqq \mathrm{Proj}_{\mathcal{C}}(\widehat{C}_t)$.
Inspired by Russo and Proutiere \cite{Russo_Proutiere_2023}, wor $\sigma > 0$, we let
\begin{align*}
\lambda^*(\sigma;C_t) \coloneqq
\argmax_{\lambda \in \P_{I}} \min_{C' \in \min(C_t)} h(\lambda, C') -\frac{  \sigma}{2} \|\lambda\|^2.
\end{align*} 
Because of the strict convexity of the regulariser term $\|\cdot\|^2$, the optimal allocation $\lambda^*(\sigma;C_t)$ is unique. We propose sampling from a convex mixture of $\lambda^*(\sigma;C_t)$ and the uniform distribution on $I$ to ensure that $N(t)/t$ converges to a strictly positive distribution as $t \to \infty$. For $\epsilon\in ( 0,1)$, let  our mixture distribution be
\begin{align}
    \lambda_\epsilon^*(\sigma;C_t) \coloneqq
(1-\epsilon) \lambda^*(\sigma;C_t) + \epsilon \cdot u , \label{eqn:mixture}
\end{align}
where $u$ is the uniform distribution on $I$, i.e., $u_{ij}=1/|I|$ for all $(i,j)\in I$. 
We use the original \textit{D-Tracking} rule in \cite{garivier2016optimal} for sampling.
\color{black}
Introducing $U_t = \{(i,j)\in I: N_{ij}(t) < (\sqrt{t}-{\binom{M}{2}} /2)_+ \}$, we sample items $(\vI_{t+1}, \vJ_{t+1})$ as    
\begin{align}
\label{eq:sample_rule_mix}
(\vI_{t+1}, \vJ_{t+1}) =
\left\{  \begin{array}{cc}
    \displaystyle\argmin_{(i,j)\in U_t} N_{ij}(t) & U_t\ne\emptyset \\
     \displaystyle\argmax_{(i,j)\in I} \,\, t\lambda_\epsilon^*(\sigma;C_t)_{ij}  - N_{ij}(t)  & U_t = \emptyset
\end{array}
 \right.  .
\end{align} 
\begin{proposition}
\label{prop:mix_optsample}
Under the sampling rule~\eqref{eq:sample_rule_mix} and under the instance $C \in \mathcal{C}$, $N(t)/t$ converges a.s.\ to $\lambda_\epsilon^*(\sigma;C) $, i.e., 
\begin{align*}
    \Pr\left(\lim_{t \rightarrow \infty} \frac{N(t)}{t} = \lambda_\epsilon^*(\sigma; C) \right)=1,
\end{align*}
where $N(t) = [N_{ij}(t): (i,j) \in I]$, $\epsilon \in (0,1)$ and $\sigma>0$.
\end{proposition}
The proof of Proposition \ref{prop:mix_optsample} can be found in Appendix~\ref{app:prf_mix_optsample}.

\subsection{Generalized Likelihood Ratio Based Stopping Rule}
\label{sec:stopping_rule}
Having established a sampling rule in the previous section, we now  outline a stopping criterion to determine when the algorithm should {\em stop} and output a clustering of the items with error probability no more than $\delta$. A natural stopping rule to consider is based on the Generalized Likelihood Ratio (GLR) test as adopted by many existing works on pure exploration problems \cite{rivera2024near, reda2021dealing, jedra2020optimal}. We show that this stopping rule is $\delta$-correct and asymptotically optimal in the limit of $\delta \to 0$.  However, the intractability of the statistic leads to difficulties in implementing the stopping rule in its native form. To address this, in Section~\ref{subsec:computationally-efficient-stopping-rule}, we propose an alternative stopping rule that is computationally efficient and achieves almost-asymptotic optimality, with an explicitly controlled gap to the lower bound.

We now introduce the GLR statistic for testing $C_t$ against its alternative instances, along with the threshold suggested in~\cite{kaufmann2021mixture}. The GLR statistic is given by
\begin{align}
\label{eq:glr}
Z(t) \coloneqq \inf_{C' \in \Alt(C_t)} \sum_{(i,j)\in I} N_{ij}(t) \, \KL((\widehat{v}_t)_{ij}, v'_{ij}),
\end{align}
where $\widehat{v}_t$ denotes the Bernoulli distribution under $\widehat{C}_t$ (see~\eqref{eq:feedback}), and the corresponding threshold is given by
$\beta(t,\delta) := 3 \sum_{(i,j)\in I} \ln (1+ \ln N_{ij}(t)) + {\binom{M}{2}} \, \mathcal{C}_{\mathrm{exp}}\big(\frac{\ln(1/\delta)}{{\binom{M}{2}}}\big)$,
where the exact mathematical expression for $\mathcal{C}_{\mathrm{exp}}(\cdot)$ is  given in \cite[Theorem~7]{kaufmann2021mixture}. Then,  
the GLR stopping rule is 
\begin{align}
\label{eq:stopping_rule}
\tau_\delta \coloneqq \inf \{t \in \mathbb{N}: Z(t) > \beta(t,\delta)\}.
\end{align} 
The below result demonstrates that the  clustering under the GLR-based   rule is incorrect with probability no more than~$\delta$. 
\begin{proposition}
\label{prop:stopping_rule_delta}
The stopping rule in~\eqref{eq:stopping_rule} is $\delta$-correct, i.e., 
\begin{align*}
\Pr \big(\tau_\delta < \infty,~ f([C_{\tau_\delta}]) \neq f([C])\big) \le\ \delta. 
\end{align*}
\end{proposition}
The proof of Proposition \ref{prop:stopping_rule_delta} can be found in Appendix \ref{app:stopping_rule_delta}.

In Theorem~\ref{thm:upper_bound}, we show   that the sampling rule  in \eqref{eq:sample_rule_mix}, together with the stopping rule in~\eqref{eq:stopping_rule}, achieves an expected sample complexity that matches with the instance-dependent lower bound $D^*(C)^{-1}$ presented in Theorem \ref{thm:lower_bound} as the parameters $\epsilon,\sigma \to 0$, and hence almost asymptotically optimal. In particular, we show that the expected stopping time grows at most logarithmically with $\delta$, and that its normalized growth rate is upper bounded by $D^*(C)^{-1}$. When taken together with the lower bound, this establishes the asymptotic optimality. 
\begin{restatable}{theorem}{UpperBound}
\label{thm:upper_bound}
For any instance $C \in \mathcal{C}$ and fixed parameters $\sigma>0$ and $\epsilon \in (0,1)$, the strategy that employs the sampling rule in \eqref{eq:sample_rule_mix} and the stopping rule in~\eqref{eq:stopping_rule} satisfies
\begin{align}
    \limsup_{\delta \rightarrow 0} \frac{\E_C[\tau_\delta]}{\ln(1/\delta)} \le D_\epsilon^*(\sigma;C)^{-1},
    \label{eq:upper-bound-GLR}
\end{align}
where the constant $0<D_\epsilon^*(\sigma;C)<\infty$ is given by
\begin{equation}
    \bigg[\inf_{C' \in \Alt(C)} \sum_{(i,j)\in I} \lambda_\epsilon^*(\sigma;C) \,\KL(v_{ij}, v'_{ij}) \bigg] - \frac{\sigma}{2}  \|\lambda_\epsilon^*(\sigma;C)\|^2.
    \label{eq:D-epsilon-sigma-star-C}
\end{equation}
By the maximum theorem \cite{ausubel1993generalized},  it can be shown that
\begin{equation}
    \lim_{ \epsilon , \sigma \to 0 } D_\epsilon^*(\sigma;C) = D^*(C).
    \label{eq:berges-maximum}
\end{equation}
\end{restatable}
The proof of Theorem \ref{thm:upper_bound} can be found in Appendix \ref{app:prf_stopping_rule}.

\subsection{A Computationally Feasible Algorithm: $\Alg$}
\label{subsec:computationally-efficient-stopping-rule}

While the GLR statistic in~\eqref{eq:glr} defines an almost asymptotically optimal stopping rule, it is computationally infeasible. The key bottleneck is the computation of the infimum in~\eqref{eq:glr}, which entails simultaneously optimizing over both $N_{ij}(t)$ and $(\widehat{c}_t)_{ij}$, leading to a large number of combinations to optimize over. While the preceding problem seems similar to the sup-inf problem in~\eqref{eqn:sup-inf}, the key difference that renders one solution inapplicable to the other are the distinct values of $(\widehat{c}_t)_{ij}$ over all $(i,j) \in I$ to minimize over in $Z(t)$ as compared to simply two values, $p$ and $q$, in $D^*(C)^{-1}$. Moreover, using $\widehat{C}_t$ (the empirical estimate of $C$) instead of its projection $C_t$ is essential for proving the $\delta$-correctness property in Proposition~\ref{prop:stopping_rule_delta}. Thus, merely substituting $(c_t)_{ij}$ in place of $(\widehat{c}_t)_{ij}$ in the objective function of \eqref{eq:glr} does not alleviate the issue at hand.

We propose a computationally efficient test statistic that preserves the $\delta$-correctness property and achieves near–asymptotic optimality, up to a multiplicative factor that can be  efficiently estimated. In essence, we obviate the complexity of the GLR statistic with respect to $N_{ij}(t)$ by ``factoring out'' the smallest count of queries, $\min_{(i',j')\in I} \, N_{i'j'}(t)$. Let
\begin{align}
\label{eq:ALTstopping_stat}
\widehat{Z}(t) \coloneqq 
 \bigg[\min_{(i',j')\in I} N_{i'j'} (t) \bigg]  \bigg[\inf_{C' \in \Alt(C_t)} \sum_{(i,j)\in I} \KL((\widehat{v}_t)_{ij}, v'_{ij}) \bigg],
\end{align}  
and the modified stopping rule 
be given by 
\begin{align}
\label{eq:ALTstopping_rule}
    \widehat{\tau}_\delta = \inf \{t \in \mathbb{N}: \widehat{Z}(t) > \beta(t,\delta)\}. 
\end{align}
Because $Z(t) \ge \widehat{Z}(t)$ a.s., the stopping rule with $\widehat{Z}(t)$ inherits the $\delta$-correctness property of Proposition \ref{prop:stopping_rule_delta}.  
 This is stated below without proof. 
\begin{restatable}{corollary}{ALTStoppingRuleDelta}
\label{prop:ALTstopping_rule_delta}
The stopping rule in~\eqref{eq:ALTstopping_rule} is $\delta$-correct, i.e., 
\begin{align*}
\Pr \big(\widehat{\tau}_\delta < \infty, ~ f([C_{\widehat{\tau}_\delta}]) \neq f([C]) \big) \le \delta. 
\end{align*}
\end{restatable}

In  Proposition \ref{prop:ALTstopping_rule_simplify}  in  Appendix~\ref{app:computationally_feas}, we derive a closed-form expression for $\widehat{Z}(t)$,
and show that the search over $\Alt(C_t)$ in~\eqref{eq:ALTstopping_stat} can be reduced to a search over finitely many equivalence classes.  Hence, the new stopping statistic in~\eqref{eq:ALTstopping_stat} becomes   a {\em computationally tractable} and  practical alternative to the GLR-based statistic $Z(t)$ (whose computation is infeasible due to lack of a closed-form expression). 
However, this  gain in computational efficient comes with a potential drawback---namely, at the  possible  expense of asymptotic optimality. To characterize how far we may have deviated from optimality, we seek an upper bound on the expected sampling complexity as previously obtained for the GLR statistic in Theorem \ref{thm:upper_bound}.

\begin{restatable}{theorem}{mixALTUpperBound}
\label{thm:mix_ALTupper_bound}
For any instance $C \in \mathcal{C}$ and sufficiently small $\sigma >0$, under the sampling rule in~\eqref{eq:sample_rule_mix} and the alternate stopping rule in~\eqref{eq:ALTstopping_rule}, the expected sample complexity satisfies
    \begin{align}
       \limsup_{\delta \rightarrow 0} \frac{\E_C[\widehat{\tau}_\delta]}{\ln(1/\delta)} \le \widetilde{D}_\epsilon^*(\sigma;C)^{-1} , 
       \label{eq:sigma_upper-bound-A3CNP}
    \end{align}
where the constant $0<\widetilde{D}_\epsilon^*(\sigma;C)<\infty$ is given by 
\begin{equation}
   \widetilde{D}_\epsilon^*(\sigma;C)= \bigg[ \min_{(i',j')\in I} {\lambda_\epsilon^*(\sigma;C)}_{i'j'} \bigg]  \bigg[ \inf_{C' \in \Alt(C)}\sum_{(i,j)\in I}\!\!\! \KL(v_{ij}, v'_{ij}) \bigg] - \frac{\sigma}{2}  \|\lambda_\epsilon^*(\sigma;C)\|^2.
    \label{eq:sigma-tilde-D-star-C}
\end{equation}
\end{restatable}  
The proof of Theorem \ref{thm:mix_ALTupper_bound} can be found in Appendix \ref{app:mix_ALTupper_bound}.

Putting together the D-Tracking sampling rule in~\eqref{eq:sample_rule_mix} and the practical stopping rule  in~\eqref{eq:ALTstopping_rule}, we summarize  our algorithm, $\Alg$: \underline{A}lmost \underline{A}symptotically Optimal \underline{A}ctive \underline{C}lustering with \underline{N}oisy \underline{P}airwise Observations.  
$\Alg$ begins by sampling every pair of items once, following which sampling proceeds according to~\eqref{eq:sample_rule_mix}. At time step $t$, the estimated instance $\widehat{C}_t$ is first projected onto $\mathcal{C}$   to obtain $C_t = \mathrm{Proj}_{\mathcal{C}}(\widehat{C}_t)$.  After solving for $\lambda^*(\sigma;C_t)$, the sampling distribution $\lambda_\epsilon^*(\sigma;C_t)$ is obtained as a   mixture of $\lambda^*(\sigma;C_t)$ and the uniform distribution as in \eqref{eqn:mixture},  and item pair 
$(\vI_{t+1}, \vJ_{t+1}) \sim \lambda_\epsilon^*(\sigma;C_t)$ is queried.
Upon receiving a binary reward from the query, the statistic $\widehat{Z}(t)$ is evaluated.
If $\widehat{Z}(t) \geq \beta(t, \delta)$, $\Alg$ stops and outputs 
$f(C_t)$, the clustering under $C_t$, at the stopping time $\widehat{\tau}_\delta=t$.

\subsection{On the Suboptimality Gap of $\Alg$ }
\label{sec:subopt_gap}

We now study the instance-dependent \emph{multiplicative} suboptimality gap incurred by replacing the GLR statistic in~\eqref{eq:glr} with the computationally feasible surrogate in $\Alg$, leading to~\eqref{eq:sigma_upper-bound-A3CNP} for fixed  $\epsilon\in(0,1)$ and $\sigma>0$. In particular, we compare  $D_\epsilon^\star(\sigma;C)$ and $\widetilde D_\epsilon^\star(\sigma;C)$ through the ratio
\begin{equation}
    {\rm SG}_{\epsilon}(\sigma;C)
    \;\coloneqq\;
    \frac{D_\epsilon^\star(\sigma;C)}{\widetilde D_\epsilon^\star(\sigma;C)}.
\end{equation}
We emphasize that the \emph{actual} gap with respect to the information-theoretic benchmark $D^\star(C)$ appearing in the lower bound in \eqref{eq:lower-bound} is
$D^\star(C)/\widetilde D_\epsilon^\star(\sigma;C)$.
Nevertheless, ${\rm SG}_\epsilon(\sigma;C)$ is a principled and convenient proxy: by~\eqref{eq:berges-maximum}, the numerator $D_\epsilon^\star(\sigma;C)$ approaches $D^\star(C)$ as $\epsilon,\sigma \to 0$, so that for sufficiently small $\epsilon$ and $\sigma$, the ratios
$D^\star(C)/\widetilde D_\epsilon^\star(\sigma;C)$ and
$D_\epsilon^\star(\sigma;C)/\widetilde D_\epsilon^\star(\sigma;C)$
are expected to be close (heuristically, and as corroborated empirically in our experiments in Section \ref{sec:experiments}).

To obtain an explicit instance-dependent upper bound on ${\rm SG}_\epsilon(\sigma;C)$, define the maximum and minimum of the sampling allocation $\lambda_\epsilon^\star(\sigma;C)$ respectively as
$\overline{m} \coloneqq \max_{(i,j)\in I}\big(\lambda_\epsilon^\star(\sigma;C)\big)_{ij}$ and 
$\underline{m} \coloneqq \min_{(i,j)\in I}\big(\lambda_\epsilon^\star(\sigma;C)\big)_{ij}$,  
and introduce the instance-dependent separation constant
\[
A(C) \;\coloneqq\; \inf_{C'\in \Alt(C)} \sum_{(i,j)\in I} \KL \big(v_{ij},v'_{ij}\big).
\]
Using the elementary bounds
$$\underline m \sum_{ i,j } \KL(v_{ij},v'_{ij})
\le \sum_{ i,j } (\lambda_\epsilon^\star)_{ij}\KL(v_{ij},v'_{ij})
\le \overline m \sum_{ i,j } \KL(v_{ij},v'_{ij})$$
together with the definitions of $D_\epsilon^\star(\sigma;C)$ and $\widetilde D_\epsilon^\star(\sigma;C)$, we obtain
\begin{equation}\label{eq:SG_basic_bound}
{\rm SG}_\epsilon(\sigma;C)
\;\le\;
\frac{\overline{m}A(C) }{\underline{m}A(C)  - \frac{\sigma}{2}\,\big\|\lambda_\epsilon^\star(\sigma;C)\big\|_2^2},
\end{equation}
whenever the denominator is positive.

It remains to control $\|\lambda_\epsilon^\star(\sigma;C)\|_2^2$. Recalling the definition of $\lambda_\epsilon^*(\sigma;C)$ from \eqref{eqn:mixture},
and using $\|\lambda^\star(\sigma;C)\|_2^2\le 1$, $\langle \lambda^\star(\sigma;C),u\rangle = 1/|I|$, and $\|u\|_2^2=1/|I|$, we have
\begin{align}
    \big\|\lambda_\epsilon^\star(\sigma;C)\big\|_2^2
    &= \big\|(1-\epsilon)\lambda^\star(\sigma;C) + \epsilon u\big\|_2^2 \nonumber\\
    &\le (1-\epsilon)^2+ \frac{2\epsilon-\epsilon^2}{|I|}. \label{eq:lambda_eps_norm_bound}
\end{align}
Substituting~\eqref{eq:lambda_eps_norm_bound} into~\eqref{eq:SG_basic_bound} yields the explicit bound
\begin{equation}\label{eq:SG_explicit_bound}
    {\rm SG}_\epsilon(\sigma;C)
    \;\le\;
    \frac{\overline{m}/\underline{m}}
    {1-\frac{\sigma}{2\underline{m}A(C) }\Big((1-\epsilon)^2+ \frac{2\epsilon-\epsilon^2}{|I|}\Big)}.
\end{equation}
The bound in~\eqref{eq:SG_explicit_bound} elucidates two  sources of the gap: the \emph{spread} $\overline m/\underline m$ of the sampling proportions, and the size of the regularization relative to $\underline m A(C) $. In particular, when $\sigma$ is small enough that the  correction term in the denominator is negligible, the bound is governed primarily by $\overline m/\underline m$.

Motivated by~\eqref{eq:SG_explicit_bound} (with $\epsilon$ and $\sigma$ small), we use the    data-dependent proxy:
\begin{align}
\widehat{\rm SG}_\epsilon(\sigma;C_t)
\;\coloneqq\;
\frac{\max_{(i,j)\in I}\big(\lambda_\epsilon^\star(\sigma;C_t)\big)_{ij}}
{\min_{(i,j)\in I}\big(\lambda_\epsilon^\star(\sigma;C_t)\big)_{ij}}. \label{eqn:data_dep}
\end{align}
When $\Alg$ is run with sufficiently small $\epsilon$ and $\sigma$, we observe that $\widehat{\rm SG}_\epsilon(\sigma; C_t)$ tracks ${\rm SG}_\epsilon(\sigma;C)$ well, providing a simple diagnostic for the dominant spread term $\overline m/\underline m$.

\begin{figure}[t!]
     \centering
 \includegraphics[width=0.5\linewidth]{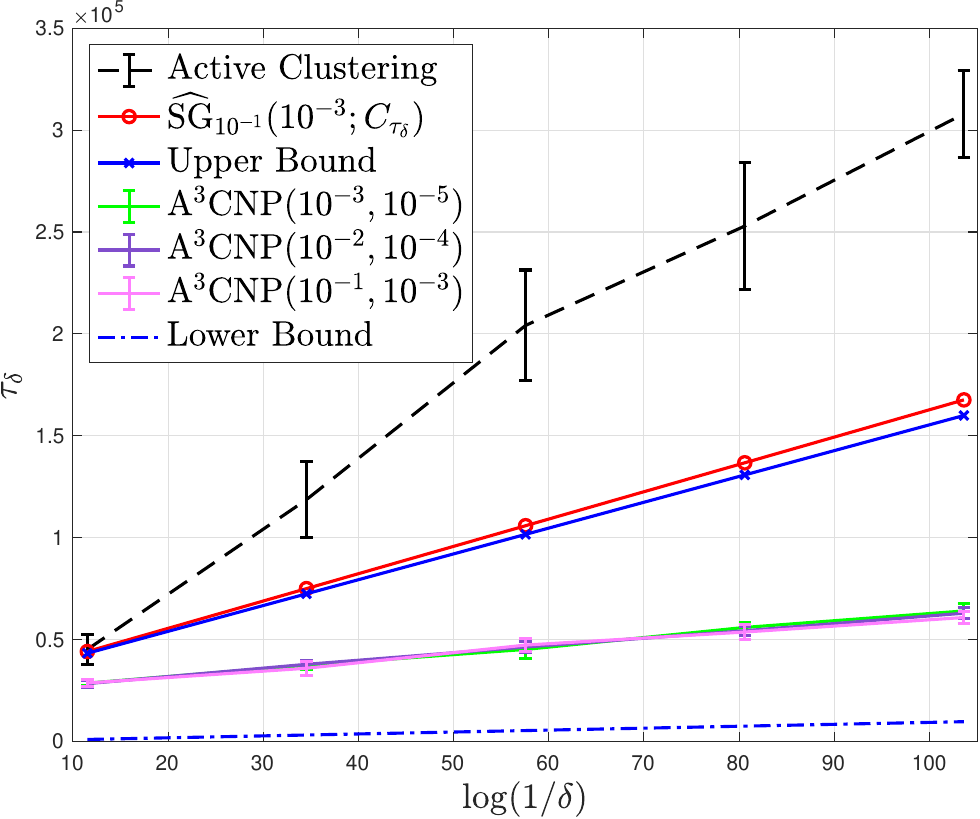}
     \caption{The asymptotic ($\delta \to 0$) sample complexity of $\Alg$, with varying $\epsilon$ (first argument) and $\sigma$ (second argument) values, relative to the active clustering algorithm of \cite{ChenVinayakHassibi2023}. Also included in the plot are the theoretical lower~\eqref{eq:lower-bound} and upper bound~\eqref{eq:sigma_upper-bound-A3CNP} and the data-dependent proxy $\widehat{\rm SG}_\epsilon(\sigma; C_{\tau_\delta})$ in~\eqref{eqn:data_dep} evaluated at the stopping time $\tau_\delta$ using $\epsilon=10^{-1}$ and $\sigma=10^{-3}$. }
     \label{fig:tns}
 \end{figure}
 
\section{Experimental Results} \label{sec:experiments}

To validate the performance of $\Alg$, we run   experiments and benchmark against the algorithm of Chen et al.~\cite{ChenVinayakHassibi2023}, which addresses the same problem setting. We work with $M=6$, $K=3$, using the difficult case where the probabilities are relatively close to each other with $p=0.6$ and $q=0.4$, and cluster assignments $\{\{1,2\}, \{3,4,5\}, \{6\}\}$. Additionally, taking cues from \cite{kaufmann2021mixture}, we use a smaller threshold of $\tilde{\beta}(t, \delta)\coloneqq 3  R \ \log (1+ \log(t/R)) +  R\ \mathcal{C}_{\mathrm{exp}} (\log ((\overline{M}-1)/\delta)/R)$ for experimentation. Here, $\overline{M}$ represents  the total number of possible clusterings of $M$ items (with an a priori unknown number of underlying clusters), and $R=M-1$ represents the {\em rank} of the problem---the worst-case number of variables that are required to describe any given equivalence class of instances;
 for a discussion on rank, see Appendix~\ref{app:sec:rank-of-active-clustering}. 
As noted in \cite[Proposition~21]{kaufmann2021mixture}, the GLR-based stopping rule, used  with the above threshold, remains $\delta$-correct. 


In Figure~\ref{fig:tns}, we compare the stopping times of $\Alg$ with those of the active clustering algorithm of~\cite{ChenVinayakHassibi2023}, averaged over $10$ independent trials. We additionally plot the lower bound from Theorem~\ref{thm:lower_bound} and the upper bound in~\eqref{eq:SG_explicit_bound}, obtained by scaling the right-hand side of~\eqref{eq:SG_explicit_bound} by $D^*(C)^{-1}\ln(1/\delta)$. Across all tested values of $\epsilon$ and $\sigma$, the slope of $\Alg$ closely matches that of the lower bound and is markedly flatter than the upper bound, demonstrating near-optimal scaling. Furthermore, the stopping time of $\Alg$ varies only marginally with $\epsilon$ and $\sigma$, underscoring its strong robustness to these parameters. The data-dependent proxy $\widehat{\rm SG}_\epsilon(\sigma;C_{\tau_\delta})$ (evaluated at the stopping time $\tau_\delta$) in~\eqref{eqn:data_dep} consistently provides an accurate approximation of the upper bound. Overall, $\Alg$ substantially and consistently outperforms the algorithm of Chen et al.~\cite{ChenVinayakHassibi2023}, yielding significantly faster (smaller) stopping times across this experimental setting.

\section{Conclusion and Future Research}
In this work, we developed a principled framework for active clustering with noisy pairwise observations in the fixed-confidence regime. By casting the problem as a pure-exploration bandit task, we derived an instance-dependent lower bound on the expected query complexity and showed how this bound can be approached by adaptive sampling and stopping rules. Our analysis culminated in the  $\Alg$ algorithm, which combines statistically efficient sampling with a computationally tractable stopping criterion, and the gap of its hardness parameter to the that of the lower bound is characterized as the target error probability tends to zero.

Beyond the specific algorithmic contributions, this paper highlights a broader perspective on active clustering as an information-acquisition problem. The characterization of the hardest alternative instances as simple merge-or-split perturbations of the true clustering suggests that fine-grained structural understanding can play a decisive role in designing efficient adaptive procedures. This viewpoint opens the door to a new class of clustering algorithms that dynamically concentrate effort on the most ambiguous structural decisions, rather than uniformly refining all pairwise relationships.

Looking ahead, several directions appear especially promising. Extending the framework to richer and more realistic feedback models—such as heterogeneous noise levels, contextual or side information, or structured query costs—would bring active clustering closer to real deployment settings. On the theoretical side, developing computationally efficient stopping rules that are fully asymptotically optimal remains an open challenge, as does understanding optimal tradeoffs in large-scale regimes where querying all pairs is infeasible. More broadly, we believe that the tools introduced here will be useful beyond clustering, offering a foundation for studying adaptive structure discovery problems at the intersection of bandits, information theory, and unsupervised learning.


\appendices

\section{Example Demonstration of Equivalence Relation (Section~\ref{sec:setup})}
\label{app:eg}
In essence, the equivalence relation defined in Section \ref{sec:setup} simply forms equivalence classes over instances that represent the same cluster assignments, but with different Bernoulli parameters for reward distributions. For example, let $C$, $C' \in \mathcal{C}$ 
\begin{align*}
  C =  \begin{pmatrix}
        p & p & q & q & q & q  \\
        p & p & q & q & q & q \\
        q & q & p & q & q & q \\
        q & q & q & p & p & p \\
        q & q & q & p & p & p 
        \\
        q & q & q & p & p & p 
\end{pmatrix}; \hspace{1cm}
C' =  \begin{pmatrix}
        p' & p' & q' & q' & q' & q'  \\
        p' & p' & q' & q' & q' & q' \\
        q' & q' & p' & q' & q' & q' \\
        q' & q' & q' & p' & p' & p' \\
        q' & q' & q' & p' & p' & p' 
        \\
        q' & q' & q' & p' & p' & p' 
\end{pmatrix},
\end{align*}
where $ 1 \ge p,p' > 0.5 > q,q \ge 0$ and $p\ne p'$, $q \ne q'$. The binary matrices $C_{=}$, $C'_{=}$, $C_{\ne}$ and $C'_{\ne}$ are as follows 
\begin{align*}
  C_{=}  = C'_{=}  \begin{pmatrix}
        1 & 1 & 0 & 0 & 0 & 0  \\
        1 & 1 & 0 & 0 & 0 & 0 \\
        0 & 0 & 1 & 0 & 0 & 0 \\
        0 & 0 & 0 & 1 & 1 & 1 \\
        0 & 0 & 0 & 1 & 1 & 1 
        \\
        0 & 0 & 0 & 1 & 1 & 1 
\end{pmatrix}; \hspace{1cm}
C_{\ne} = C'_{\ne} =  \begin{pmatrix}
        0 & 0 & 1 & 1 & 1 & 1  \\
        0 & 0 & 1 & 1 & 1 & 1 \\
        1 & 1 & 0 & 1 & 1 & 1 \\
        1 & 1 & 1 & 0 & 0 & 0 \\
        1 & 1 & 1 & 0 & 0 & 0 
        \\
        1 & 1 & 1 & 0 & 0 & 0 
\end{pmatrix}.  
\end{align*}
This can be observed easily by replace all $p$, $p'$ (resp.\ $q$, $q'$) with $1$ (resp.\ $0$) in $C_{=}$ and $C'_{=}$ and vice versa for $C_{\ne}$ and $C'_{\ne}$. Then, $C\sim C'$ and both matrices represent cluster assignments $\{ \{1,2\}, \{3\}, \{4,5,6\}\}$.

\section{Proofs of Results in  Section \ref{sec:lower_bound}}

\subsection{Proof of Theorem \ref{thm:lower_bound}} \label{app:prf_lower_bd}


\begin{proof}
 We begin by assuming that $\E_C [\tau_\delta]$ is finite and from a simplification of the sup-inf optimization problem in Appendix~\ref{app:proof_inner_inf}, it will be easy to see that $D^*(C)>0$. For fixed $C \in \mathcal{C}$ and $C' \in \Alt(C)$, we can consider them as bandit models with ${\binom{M}{2}}$ arms, $\{v_{ij}\}_{(i,j)\in I}$ and $\{v'_{ij}\}_{(i,j)\in I}$ respectively. Then, by \cite[Lemma 1]{kaufmann2016complexity}, we have
\begin{align*}
\sum_{(i,j) \in I} \E_C [ N_{ij}(\tau_\delta)] \KL(v_{ij}, v'_{ij}) 
\ge d(\delta, 1-\delta). 
\end{align*}
Since $\sum_{(i,j) \in I} \E_C [ N_{ij}(\tau_\delta)] = \E_C [ \tau_\delta]$ and $\{\frac{\E_C [ N_{ij}(\tau_\delta)]}{\E_C [ \tau_\delta]}\}_{(i,j)\in I}$ forms a probability distribution over~$I$, 
\begin{align*}
 d(\delta, 1-\delta) &\le \inf_{C' \in \Alt(C)} \sum_{(i,j) \in I} \E_C [ N_{ij}(\tau_\delta)] \KL(v_{ij}, v'_{ij}) \\
&= \E_C [ \tau_\delta] \inf_{C' \in \Alt(C)} \sum_{(i,j) \in I} \frac{\E_C [ N_{ij}(\tau_\delta)]}{\E_C [ \tau_\delta]} \KL(v_{ij}, v'_{ij}) \\
&\le \E_C [ \tau_\delta] \sup_{\lambda \in \P_{I}} \inf_{C' \in \Alt(C)} \sum_{(i,j) \in I} \lambda_{ij} \KL(v_{ij}, v'_{ij}) \\
&= \E_C [ \tau_\delta] D^*(C). 
\end{align*} 
If $\E_C [ \tau_\delta]$ is infinite, the inequality holds trivially. Then, as $\lim_{\delta \to 0} \frac{d(\delta, 1-\delta)}{\ln(1/\delta)} = 1$, 
\begin{align*}
\liminf_{\delta \rightarrow 0}\frac{\E_C[\tau_\delta]}{\ln(1/\delta)} \ge D^*(C)^{-1} 
\end{align*}
holds, as desired. 
\end{proof}

\subsection{Proof of Theorem \ref{thm:inner_inf}}
\label{app:proof_inner_inf}


\begin{proof}
Using the equivalence relation defined in Section~\ref{sec:setup}, we can write the sup-inf optimization problem as
\begin{align*}
    \sup_{\lambda \in \P_{I}} \inf_{C' \in \Alt(C)} h(\lambda, C') = \sup_{\lambda \in \P_{I}} \inf_{[C'] \in \Alt(C)/\sim}  \inf_{C'' \in [C']} h(\lambda, C''), 
\end{align*}
where recall that $\Alt(C)/\sim$ denotes the set of all equivalence classes $\Alt(C)$.
By Lemma \ref{lem:pqstar},
\begin{align*}
&\sup_{\lambda \in \P_{I}} \inf_{[C'] \in \Alt(C)/\sim}  \inf_{C'' \in [C']} h(\lambda, C'') \\ &= \sup_{\lambda \in \P_{I}} \inf_{[C'] \in \Alt(C)/\sim} \alpha   -b_1 \log (1-q') - b_2 \log (q')   - b_3 \log (1-p')  -  b_4 \log (p')
\\ &= \sup_{\lambda \in \P_{I}} \min_{[C'] \in \Alt(C)/\sim} z(\lambda_{N_1}, \lambda_{N_2}), 
\end{align*}
where $\alpha, b_1, b_2, b_3, b_4, p'$ and $q'$ are as defined in the lemma and the inf can be written as a min since $\Alt(C)/\sim$ is finite. The objective function above is equivalent to $z(\lambda_{N_1}, \lambda_{N_2})$ as defined in Lemma \ref{lem:inc_z} (below) and is increasing in $\lambda_{N_1}, \lambda_{N_2}$ for a fixed $\lambda$ and $C \in \mathcal{C}$, where $\lambda_{N_1} = \sum_{(i,j)\in \Lambda_{N_1}([C'])} \lambda_{ij}$ and $\lambda_{N_2} = \sum_{(i,j)\in \Lambda_{N_2}([C'])} \lambda_{ij}$. 

Then, we analyze the inner minimum 
\begin{align*}
 \min_{[C'] \in \Alt(C)/\sim} z(\lambda_{N_1}, \lambda_{N_2}),.
\end{align*} 
In this case, when minimizing $\lambda_{N_1} + \lambda_{N_2}$ for any arbitrary $\lambda \in \P_{I}$ over $\Lambda_{N_1}([C']),\Lambda_{N_2}([C'])$ that are always disjoint, we must have 
\begin{align*}
&\lambda_{N_1} + \lambda_{N_2} \le \lambda_{N'_1} + \lambda_{N'_2}= \sum_{(i,j)\in \Lambda_{N_1}([C''])} \lambda_{ij} +  \sum_{(i,j)\in \Lambda_{N_2}([C''])} \lambda_{ij} \\
&\overset{\text{(i)}}{\iff} \Lambda_{N_1}([C']) \subset \Lambda_{N_1}([C'']) \text{ and } \Lambda_{N_2}([C']) \subset \Lambda_{N_2}([C'']) \\
&\overset{\text{(ii)}}{\iff} \lambda_{N_1} \le \lambda_{N'_1} \text{ and } \lambda_{N_2} \le \lambda_{N'_2} \\
&\overset{\text{(iii)}}{\Rightarrow} z(\lambda_{N_1},\lambda_{N_2}) \le z(\lambda_{N'_1},\lambda_{N'_2}).
\end{align*}

Implications (ii) and (iii) are obvious from the definitions of $\Lambda_{N_1}([C'])$, $\Lambda_{N_2}([C'])$ and the fact that $z$ is increasing in $\lambda_{N_1}$, $\lambda_{N_2}$. From (ii), $\Leftarrow$ of (i) is clear. So, we simply need to show $\Rightarrow$ of (i). Suppose $\lambda_{N_1} + \lambda_{N_2} \le \lambda_{N'_1} + \lambda_{N'_2}$ but $\Lambda_{N_1}([C']) \not\subset \Lambda_{N_1}([C'']) \text{ or } \Lambda_{N_2}([C']) \not\subset \Lambda_{N_2}([C''])$. Without loss of generality, assume that $\Lambda_{N_1}([C']) \not\subset \Lambda_{N_1}([C''])$ so that we can find $(i,j) \in \Lambda_{N_1}([C']) \setminus (\Lambda_{N_1}([C']) \cap \Lambda_{N_1}([C'']))$. Then, there exists $\lambda' \in \P_{I}$ and small enough $\epsilon$, $\gamma > 0$ such that 
\begin{align*}
\lambda'_{N'_1} + \lambda'_{N'_2}=\epsilon >\lambda'_{\Lambda_{N_1}([C']) \cap \Lambda_{N_1}([C''])} + \lambda'_{\Lambda_{N_2}([C']) \cap \Lambda_{N_2}([C''])}
&\ge 0,\\
\lambda'_{N_p\setminus (\Lambda_{N_1}([C']) \cup \Lambda_{N_1}([C'']))}+\lambda'_{N_q\setminus (\Lambda_{N_2}([C']) \cup \Lambda_{N_2}([C'']))}=\gamma &> 0
\end{align*}
and 
\begin{align*}
\lambda'_{N_1} + \lambda'_{N_2} = 1 - \gamma - \epsilon + \lambda'_{\Lambda_{N_1}([C']) \cap \Lambda_{N_1}([C''])} + \lambda'_{\Lambda_{N_2}([C']) \cap \Lambda_{N_2}([C''])} > \epsilon.
\end{align*}
 Hence, a contradiction. 

Due to implications 1, 2 and 3 above, it should be clear that for an arbitrarily fixed $\lambda \in \P_{I}$ and $C \in \mathcal{C}^K$, minimizing $\lambda_{N_1}+\lambda_{N_2}$ leads to minimizing $z(\lambda_{N_1},\lambda_{N_2})$. We will make use of this fact repeatedly in Lemma~\ref{lem:K+1K-1}, enabling a further simplification of the inner minimum,
\begin{align*}
\sup_{\lambda \in \P_{I}} \min_{[C'] \in \Alt(C)/\sim} z(\lambda_{N_1}, \lambda_{N_2}) = \sup_{\lambda \in \P_{I}} \min_{[C'] \in (\Alt(C) \cap (\mathcal{C}^{K-1} \cup \mathcal{C}^{K+1}))/\sim} z(\lambda_{N_1}, \lambda_{N_2}). 
\end{align*}

Finally, to conclude, we need to show that for $[C'] \in (\Alt(C) \cap (\mathcal{C}^{K-1} \cup \mathcal{C}^{K+1}))/\sim$, if $[C'] \notin \Alt(C)_1/\sim$, then there exists a $[\bar{C}] \in \Alt(C)_1/\sim$ such that $\lambda_{\bar{N_1}} +\lambda_{\bar{N_2}}= \sum_{(i,j)\in \Lambda_{N_1}([\bar{C}])} \lambda_{ij}+\sum_{(i,j)\in \Lambda_{N_2}([\bar{C}])} \lambda_{ij} \le  \lambda_{N_1} +\lambda_{N_2}$. Following the approach of Lemma \ref{lem:K+1K-1}, we will consider 2 cases: first, when $[C']$ has $K+1$ clusters; and second, when $[C']$ has $K-1$ clusters. \\
\underline{Case 1: $[C'] \in (\Alt(C) \cap \mathcal{C}^{K+1})/\sim$} \\
If $[C'] \notin \Alt(C)_1/\sim$, then 
\[
f([C']) \ne \{g_1,\cdots,g_{i-1},g'_{i_1},g'_{i_2},g_{i+1},\cdots,g_K\},
\] where $g'_{i_1}\cup g'_{i_2}=g_i$ and there exists at least one $g'_j \not\subset g_k$ for all $k \in [K]$.\footnote{If $g'_j \subset g_k$ for all $j \in [K+1]$, then there cannot be more than two $g'_j$ that are proper subsets of some $g_k$, forcing $[C'] \in \Alt(C)_1$.} This implies that there exists items $m_1 > m_2 \in g'_j$ such that $m_1$ and $m_2$ are from different $g_k$. Hence, $(m_2, m_1) \in \Lambda_{N_2}([C'])$. Further, since $|f([C'])|=K+1$, we also cannot have $g_i \subset g'_k$ for all $g_i \in f(C)$ because then we would have less than $K+1$ sets in $f([C'])$. Then, there must exist at least one $g_i \in f(C)$ such that there exists $g'_{k_1}, g'_{k_2} \in f([C'])$ where $|g'_{k_1} \cap g_i|, |g'_{k_2} \cap g_i|>0$. Then, let $g_{k_1,i} \coloneqq g'_{k_1} \cap g_i$ and $I_{k_1,i}\coloneqq \{(i,j) \in g_{k_1,i} \times (g_i \setminus g_{k_1,i}) \cup (g_i \setminus g_{k_1,i}) \times g_{k_1,i}: j > i\} \subset \lambda_{N_1}(C')$. Since $[\bar{C}] \coloneqq f^{-1}(\{g_1,\cdots,g_{i-1},g_{k_1,i},g_i \setminus g_{k_1,i},g_{i+1},\cdots,g_K\}) \in \Alt(C)_1$ and $\lambda_{\bar{N_1}}+\lambda_{\bar{N_2}}=\sum_{(i,j)\in I_{k_1,i}} \lambda_{ij}$, we must have $\lambda_{\bar{N_1}}+\lambda_{\bar{N_2}} \le \lambda_{N_1}+\lambda_{N_2}$. \\

\underline{Case 2: $[C'] \in (\Alt(C) \cap \mathcal{C}^{K-1})/\sim$} \\
If $[C'] \notin \Alt(C)_1/\sim$, then 
\[
f([C']) \ne \{g_1,\cdots,g_{i_1-1},g_{i_1+1},\cdots,g_{i_2-1},g_{i_2+1},\cdots,g_K,g'_i\},
\] where $g'_{i} = g_{i_1}\cup g_{i_2}$ and there exists at least one $g_j \not\subset g'_k$, for all $k\in [K-1]$.\footnote{
If for all $g_j\in f(C)$, $g_j \subset g'_k$, then none of the $g_j$ has been split and in order for $|f([C'])|=K-1$, we must have exactly two $g_{i_1},g_{i_2} \in f(C)$ merged to form $g'_i$. Hence, $[C'] \in \Alt(C)_1/\sim$. 
} This implies that there is some $g'_k\in f([C'])$ such that $|g_j|>|g_j\cap g'_k|>0$ and letting $g_{k,j} \coloneqq g'_{k} \cap g_j$ and $I_{k,j}\coloneqq \{(i,j) \in g_{k,j} \times (g_j \setminus g_{k,j}) \cup (g_j \setminus g_{k,j}) \times g_{k,j}: j > i\} \subset \Lambda_{N_1}([C'])$. 
Then, for $[\bar{C}] \coloneqq f^{-1}(\{g_1,\cdots,g_{j-1},g_{k,j},g_j \setminus g_{k,j},g_{j+1},\cdots,g_K\}) \in \Alt(C)_1$, we have $\lambda_{\bar{N_1}}+\lambda_{\bar{N_2}}=\sum_{(i,j)\in I_{k,j}} \lambda_{ij} \le \lambda_{N_1}+\lambda_{N_2}$. 
\end{proof}
\begin{lemma} 
\label{lem:pqstar}
For a fixed $\lambda \in \P_{I}$, $C \in \mathcal{C}$ and any $[C']$, solving $$\inf_{C'' \in [C']}
h(\lambda,C'') = \inf_{1 \ge p''>0.5>q''\ge0}
h(\lambda,C'')$$ yields
\begin{align*}
p'' =  \max \Big\{0.5,  \frac{\lambda_{N_1}'}{\lambda_{N_1}'+\lambda_{N_2}}p + \frac{\lambda_{N_2}}{\lambda_{N_1}'+\lambda_{N_2}}q \Big\}; \quad q'' =  \min \Big\{0.5,  \frac{\lambda_{N_1}}{\lambda_{N_1}+\lambda_{N_2}'}p + \frac{\lambda_{N_2}'}{\lambda_{N_1}+\lambda_{N_2}'}q \Big\},
\end{align*}
where with a slight abuse of notation, $\lambda_{N_1} = \sum_{(i,j)\in \Lambda_{N_1}([C'])} \lambda_{ij}$, $\lambda'_{N_1} = \sum_{(i,j)\in N_p} \lambda_{ij} -\lambda_{N_1}$, $\lambda_{N_2} = \sum_{(i,j)\in \Lambda_{N_2}([C'])} \lambda_{ij}$ and $\lambda'_{N_2} = \sum_{(i,j)\in N_q} \lambda_{ij} -\lambda_{N_2}$.  
\end{lemma}
\begin{proof} Recall the definition of $h(\lambda,C)$ in~\eqref{eqn:sup-inf}. 
Since for all $C'', C''' \in [C']$, $\Lambda_{N_1}(C'')=\Lambda_{N_1}(C''')$ and $\Lambda_{N_2}(C'')=\Lambda_{N_2}(C''')$, $\lambda_{N_1}, \lambda_{N_2}, \lambda'_{N_1}$ and $\lambda'_{N_2}$ are constant. Then, solving the optimization problem in the lemma simply reduces to solving for the optimal $p''$ and $q''$. Expanding $h(\lambda, C'')$, we have
\begin{align*}
h(\lambda, C'') &=  \lambda_{N_1} d(p,q'') +\lambda_{N_2}'  d(q,q'') + \lambda_{N_2} d(q,p'') + \lambda_{N_1}'  d(p,p'')  \\
&= \lambda_{N_1}  \left[ (1-p) \log \left( \frac{1-p}{1-q''} \right) +  p \log \left( \frac{p}{q''} \right) \right] +  \lambda_{N_2}' \left[ (1-q) \log \left( \frac{1-q}{1-q''} \right) +  q \log \left( \frac{q}{q''} \right)\right] \\
&\quad +  \lambda_{N_2} \left[ (1-q) \log \left( \frac{1-q}{1-p''} \right) +  q \log \left( \frac{q}{p''} \right) \right] +  \lambda_{N_1}' \left[ (1-p) \log \left( \frac{1-p}{1-p''} \right) +  p \log \left( \frac{p}{p''} \right)\right] \\
&= \lambda_{N_1} (1-p) \log (1-p) - \lambda_{N_1} (1-p) \log (1-q''') +  \lambda_{N_1}  p \log (p) - \lambda_{N_1}  p \log (q''') \\
&\quad +  \lambda_{N_2}'  (1-q) \log (1-q)- \lambda_{N_2}'  (1-q) \log (1-q''') +  \lambda_{N_2}' q \log (q) -  \lambda_{N_2}' q \log(q''') \\
&\quad + \lambda_{N_2} (1-q) \log (1-q)  - \lambda_{N_2} (1-q) \log (1-p'') + \lambda_{N_2} q \log(q)  - \lambda_{N_2} q \log (p'') \\
&\quad +  \lambda_{N_1}' (1-p) \log (1-p) - \lambda_{N_1}' (1-p) \log (1-p'') +  \lambda_{N_1}'  p \log (p) -  \lambda_{N_1}'  p \log  (p'') \\
&=  (\lambda_{N_1}+\lambda'_{N_1})[ (1-p) \log (1-p) +  p \log (p) ] +  (\lambda_{N_2}+\lambda_{N_2}') [ (1-q) \log (1-q)  +  q \log (q)] \\
&\quad  -[ \lambda_{N_1} (1-p)  + \lambda_{N_2}'  (1-q)] \log (1-q'') - [\lambda_{N_1}  p+  \lambda_{N_2}' q] \log (q'') \\
& \quad - [ \lambda_{N_1}' (1-p) + \lambda_{N_2} (1-q)] \log (1-p'')  -  [\lambda_{N_1}'  p  + \lambda_{N_2} q ]\log (p'') \\
&= \alpha   -b_1 \log (1-q'') - b_2 \log (q'')   - b_3 \log (1-p'')  -  b_4 \log (p''), 
\end{align*}
where $\alpha\coloneqq (\lambda_{N_1}+\lambda'_{N_1})[ (1-p) \log (1-p) +  p \log (p) ] +  (\lambda_{N_2}+\lambda_{N_2}') [ (1-q) \log (1-q)  +  q \log (q)]$, $b_1\coloneqq  \lambda_{N_1} (1-p)  + \lambda_{N_2}'  (1-q)$, $b_2\coloneqq \lambda_{N_1}  p+  \lambda_{N_2}' q$, $b_3 \coloneqq \lambda_{N_1}' (1-p) + \lambda_{N_2} (1-q)$ and $b_4\coloneqq \lambda_{N_1}'  p  + \lambda_{N_2} q $ are constants. 

Defining $g_2(q'') \coloneqq -b_1\log (1-q'') - b_2 \log (q'') $ and $g_4(p'') \coloneqq  - b_3 \log (1-p'')  -  b_4 \log (p'') $ and re-labeling $p''$ to $y$ and $q''$ to $x$ for easy reading, we differentiate $h$ with respect to $x$ and $y$ to get 
\begin{align}
\frac{dh}{dx} &= \frac{dg_2}{dx} =  \frac{b_1}{1-x} - \frac{b_2}{x}  = (b_1+b_2)\frac{x-\frac{b_2}{b_1+b_2}}{x(1-x)} \label{eqn:dhdx} \\
\frac{dh}{dy} &= \frac{dg_4}{dy} =  \frac{b_3}{1-y} - \frac{b_4}{y}  =(b_3+b_4) \frac{y-\frac{b_4}{b_3+b_4}}{y(1-y)}. \label{eqn:dhdy}
\end{align} 
Since $b_1, b_2, b_3, b_4 \ge 0$ and $ 1/2> x \ge 0$ and  $1 \ge y > 1/2 $: \\
Inspecting \eqref{eqn:dhdx}:\\

\underline{Case 1: } $\frac{b_2}{b_1+b_2} \ge1/2$ \\

then, $\frac{dh}{dx} < 0$, $\forall x \in (0,1/2)$ and the infimum occurs at $x=1/2$\\

\underline{Case 2: } $0 < \frac{b_2}{b_1+b_2} < 1/2$ \\ 

then, $\frac{dh}{dx} < 0$, $\forall x \in (0,\frac{b_2}{b_1+b_2})$ and $\frac{dh}{dx} > 0$, $\forall x \in (\frac{b_2}{b_1+b_2},1/2)$ so, the infimum occurs at $x=\frac{b_2}{b_1+b_2}$\\

\underline{Case 3: } $b_2=0, b_1 >0 \implies \frac{b_2}{b_1+b_2}=0$ \\ 

then, $\frac{dh}{dx}=b_1 \frac{1}{1-x} > 0$, $\forall x \in (0,1/2)$ so the infimum occurs at $x=0=\frac{b_2}{b_1+b_2}$. \\

\underline{Case 4: } $b_2>0, b_1 =0 \implies \frac{b_2}{b_1+b_2}=1$ \\

then, $\frac{dh}{dx}=b_2 \frac{-1}{x} < 0$, $\forall x \in (0,1/2)$ so the infimum occurs at $x=0.5$. \\

\underline{Case 5: } $b_2=b_1 =0 \implies g_2(x)=0$.  \\

Inspecting \eqref{eqn:dhdy}: \\

\underline{Case 1: } $0 < \frac{b_4}{b_4+b_3} \le 1/2$ \\

then, $\frac{dh}{dy} > 0$, $\forall y \in (1/2,1)$ and the infimum occurs at $y=1/2$\\

\underline{Case 2: } $\frac{b_4}{b_4+b_3} > 1/2$ \\ 

then, $\frac{dh}{dy} < 0$, $\forall y \in (1/2,\frac{b_4}{b_3+b_4})$ and $\frac{dh}{dy} > 0$, $\forall y \in (\frac{b_4}{b_3+b_4},1]$ so, the infimum occurs at $y=\frac{b_4}{b_3+b_4}$\\

\underline{Case 3: } $b_4=0, b_3 >0 \implies \frac{b_4}{b_3+b_4}=0$ \\ 

then, $\frac{dh}{dy}=b_3 \frac{1}{1-y} > 0$, $\forall y \in (1/2,1)$ so the infimum occurs at $y=0.5$. \\

\underline{Case 4: } $b_4>0, b_3 =0 \implies \frac{b_4}{b_3+b_4}=1$ \\

then, $\frac{dh}{dx}=b_4 \frac{-1}{y} < 0$, $\forall y \in (1/2,1)$ so the infimum occurs at $y=1= \frac{b_4}{b_3+b_4}$. \\

\underline{Case 5: } $b_3=b_4 =0 \implies g_4(y)=0$.  \\

Putting all the cases together, we have \\
\begin{align*}
q'' = x &= \min \Big\{0.5, \frac{b_2}{b_1+b_2}\Big\} =  \min \Big\{0.5,  \frac{\lambda_{N_1}}{\lambda_{N_1}+\lambda_{N_2}'}p + \frac{\lambda_{N_2}'}{\lambda_{N_1}+\lambda_{N_2}'}q \Big\} \\
p'' = y &= \max \Big\{0.5, \frac{b_4}{b_3+b_4}\Big\}=  \max \Big\{0.5,  \frac{\lambda_{N_1}'}{\lambda_{N_1}'+\lambda_{N_2}}p + \frac{\lambda_{N_2}}{\lambda_{N_1}'+\lambda_{N_2}}q \Big\}, 
\end{align*} where the convention is that if $b_1=b_2=0$, $\frac{b_2}{b_1+b_2}=0$ and similarly for $b_3=b_4=0$. It is not possible that $b_1=b_2=b_3=b_4=0$ since $\sum_{(i,j) \in I} \lambda_{ij}=1$. \end{proof}

\begin{lemma}
\label{lem:inc_z}
For a fixed $\lambda \in \P_{I}$, $C \in \mathcal{C}$, we can define a function $z: \R^2 \rightarrow \R$ where 
\begin{align*}
z(\lambda_{N_1},\lambda_{N_2}) &= \alpha - b_1(\lambda_{N_1},\lambda_{N_2})\log(1-q'(\lambda_{N_1},\lambda_{N_2}))\\
    &\quad - b_2(\lambda_{N_1},\lambda_{N_2})\log(q'(\lambda_{N_1},\lambda_{N_2}))\\
    &\quad - b_3(\lambda_{N_1},\lambda_{N_2})\log(1-p'(\lambda_{N_1},\lambda_{N_2}))\\
    &\quad - b_4(\lambda_{N_1},\lambda_{N_2})\log(p'(\lambda_{N_1},\lambda_{N_2})).
\end{align*}
 where
\begin{align*}
p' =  \max \{0.5,  \frac{b_4}{b_3+b_4} \}&; \quad q' =  \min \{0.5,  \frac{b_2}{b_1+b_2} \} \\
b_1=  \lambda_{N_1} (1-p)  + \lambda_{N_2}'  (1-q)&; \quad b_2= \lambda_{N_1}  p+  \lambda_{N_2}' q \\ b_3 = \lambda_{N_1}' (1-p) + \lambda_{N_2} (1-q)&; \quad b_4= \lambda_{N_1}'  p  + \lambda_{N_2} q 
\end{align*}  and $\lambda_{N_1}' =\sum_{(i,j) \in N_p} \lambda_{ij}-\lambda_{N_1}$, $\lambda_{N_2}' =\sum_{(i,j) \in N_q} \lambda_{ij}-\lambda_{N_2}$. $z$ is increasing on $D=\{(\lambda_{N_1},\lambda_{N_2})\in \R^2: p' > 0.5 \text{ or }  q' < 0.5 \}$.
\end{lemma}

\begin{proof}
Since $\frac{d \lambda_{N_2}'}{d \lambda_{N_2}}=-1$ and $\frac{b_2}{b_1+b_2}=\frac{\lambda_{N_1}}{\lambda_{N_1}+\lambda_{N_2}'}p + \frac{\lambda_{N_2}'}{\lambda_{N_1}+\lambda_{N_2}'}q$, if $b_2 < b_1$, 
\begin{align*}
\frac{d q'}{d \lambda_{N_1}} &= \frac{p}{\lambda_{N_1}+\lambda_{N_2}'} - \frac{\lambda_{N_1}p}{(\lambda_{N_1}+\lambda_{N_2}')^2} - \frac{\lambda_{N_2}'q}{(\lambda_{N_1}+\lambda_{N_2}')^2} \\
&= \frac{\lambda_{N_1}p+\lambda_{N_2}'p- \lambda_{N_1}p-\lambda_{N_2}'q}{(\lambda_{N_1}+\lambda_{N_2}')^2}\\
&= \frac{\lambda_{N_2}'(p-q)}{(\lambda_{N_1}+\lambda_{N_2}')^2} \\
\frac{d q'}{d \lambda_{N_2}} &= \frac{-q}{\lambda_{N_1}+\lambda_{N_2}'} + \frac{\lambda_{N_2}'q}{(\lambda_{N_1}+\lambda_{N_2}')^2} + \frac{\lambda_{N_1}p}{(\lambda_{N_1}+\lambda_{N_2}')^2} \\
&= \frac{-\lambda_{N_1}q-\lambda_{N_2}'q+\lambda_{N_2}'q+\lambda_{N_1}p}{(\lambda_{N_1}+\lambda_{N_2}')^2}\\
&= \frac{\lambda_{N_1}(p-q)}{(\lambda_{N_1}+\lambda_{N_2}')^2}
\end{align*} 
If $b_2 \ge b_1$, $\frac{d q'}{d \lambda_{N_1}}=\frac{d q'}{d \lambda_{N_2}} = 0$. \\
Since $\frac{d \lambda_{N_1}'}{d \lambda_{N_1}}=-1$ and $\frac{b_4}{b_3+b_4}=\frac{\lambda_{N_1}'}{\lambda_{N_1}'+\lambda_{N_2}}p + \frac{\lambda_{N_2}}{\lambda_{N_1}'+\lambda_{N_2}}q$, if $b_4 > b_3$, 
\begin{align*}
\frac{d p'}{d \lambda_{N_1}} &= \frac{-p}{\lambda_{N_1}'+\lambda_{N_2}} + \frac{\lambda_{N_1}'p}{(\lambda_{N_1}'+\lambda_{N_2})^2} + \frac{\lambda_{N_2}q}{(\lambda_{N_1}'+\lambda_{N_2})^2} \\
&= \frac{-\lambda_{N_1}'p-\lambda_{N_2}p+ \lambda_{N_1}'p+\lambda_{N_2}q}{(\lambda_{N_1}'+\lambda_{N_2})^2}\\
&= \frac{\lambda_{N_2}(q-p)}{(\lambda_{N_1}'+\lambda_{N_2})^2}\\
\frac{d p'}{d \lambda_{N_2}} &= \frac{q}{\lambda_{N_1}'+\lambda_{N_2}} - \frac{\lambda_{N_2}q}{(\lambda_{N_1}'+\lambda_{N_2})^2} - \frac{\lambda_{N_1}'p}{(\lambda_{N_1}'+\lambda_{N_2})^2} \\
&= \frac{\lambda_{N_1}'q+\lambda_{N_2}q- \lambda_{N_2}q-\lambda_{N_1}'p}{(\lambda_{N_1}'+\lambda_{N_2})^2}\\
&= \frac{\lambda_{N_1}'(q-p)}{(\lambda_{N_1}'+\lambda_{N_2})^2}
\end{align*} 
If $b_4 \le b_3$, $\frac{d p'}{d \lambda_{N_1}} =\frac{d p'}{d \lambda_{N_2}} = 0$. \\

Then,
\begin{alignat*}{3}
\frac{d q'}{d \lambda_{N_1}} &=   \begin{cases}
        \frac{\lambda_{N_2}'(p-q)}{(\lambda_{N_1}+\lambda_{N_2}')^2} \ge 0 , & \text{if $b_1-b_2 > 0$} \\
         0, & \text{otherwise} \\
            \end{cases};&\qquad&
  \frac{d q'}{d \lambda_{N_2}} &=   \begin{cases}
        \frac{\lambda_{N_1}(p-q)}{(\lambda_{N_1}+\lambda_{N_2}')^2} \ge 0 , & \text{if $b_1-b_2 > 0$} \\
         0, & \text{otherwise} \\
            \end{cases}\\
              \frac{d p'}{d \lambda_{N_1}}&=   \begin{cases}
      \frac{\lambda_{N_2}(q-p)}{(\lambda_{N_1}'+\lambda_{N_2})^2} \le 0 , & \text{if $b_3-b_4 < 0$} \\
         0, & \text{otherwise} \\
            \end{cases}; &\qquad&
            \frac{d p'}{d \lambda_{N_2}}&=   \begin{cases}
     \frac{\lambda_{N_1}'(q-p)}{(\lambda_{N_1}'+\lambda_{N_2})^2} \le 0 , & \text{if $b_3-b_4 < 0$} \\
         0, & \text{otherwise} \\
            \end{cases}
\end{alignat*}

Since 
\begin{align*}
\frac{d b_1}{d \lambda_{N_1}} &= (1-p); \quad \frac{d b_1}{d \lambda_{N_2}} = -(1-q)\\
\frac{d b_2}{d \lambda_{N_1}} &= p; \quad \frac{d b_2}{d \lambda_{N_2}} = -q\\
\frac{d b_3}{d \lambda_{N_1}} &= -(1-p); \quad \frac{d b_3}{d \lambda_{N_2}} = (1-q)\\
\frac{d b_4}{d \lambda_{N_1}} &= -p; \quad \frac{d b_4}{d \lambda_{N_2}} = q,
\end{align*}
we have
\begin{align}
\nonumber
\frac{d z(\lambda_{N_1},\lambda_{N_2})}{d\lambda_{N_1}}=&-\frac{d b_1}{d \lambda_{N_1}} \log \{ 1-q'(\lambda_{N_1},\lambda_{N_2}) \} -  b_1\frac{d \log \{ 1-q'(\lambda_{N_1},\lambda_{N_2}) \}}{d \lambda_{N_1}} \\ \nonumber
&-\frac{d b_2}{d \lambda_{N_1}} \log \{ q'(\lambda_{N_1},\lambda_{N_2}) \} -  b_2\frac{d \log \{q'(\lambda_{N_1},\lambda_{N_2}) \}}{d \lambda_{N_1}} \\ \nonumber
&-\frac{d b_3}{d \lambda_{N_1}} \log \{ 1-p'(\lambda_{N_1},\lambda_{N_2}) \} -  b_3\frac{d \log \{ 1-p'(\lambda_{N_1},\lambda_{N_2}) \}}{d \lambda_{N_1}} \\ \nonumber
&-\frac{d b_4}{d \lambda_{N_1}} \log \{ p'(\lambda_{N_1},\lambda_{N_2}) \} -  b_4\frac{d \log \{p'(\lambda_{N_1},\lambda_{N_2}) \}}{d \lambda_{N_1}} \\ \nonumber
=&-(1-p) \log \{ 1-q'(\lambda_{N_1},\lambda_{N_2}) \} +  b_1 (1-q'(\lambda_{N_1},\lambda_{N_2}))^{-1} \frac{d q'(\lambda_{N_1},\lambda_{N_2}) }{d \lambda_{N_1}} \\ \nonumber
&-p \log \{ q'(\lambda_{N_1},\lambda_{N_2}) \} -  b_2(q'(\lambda_{N_1},\lambda_{N_2}))^{-1}\frac{dq'(\lambda_{N_1},\lambda_{N_2}) }{d \lambda_{N_1}} \\ \nonumber
&+(1-p) \log \{ 1-p'(\lambda_{N_1},\lambda_{N_2}) \} +  b_3(1-p'(\lambda_{N_1},\lambda_{N_2}))^{-1} \frac{d p'(\lambda_{N_1},\lambda_{N_2}) }{d \lambda_{N_1}} \\ \nonumber
&+p \log \{ p'(\lambda_{N_1},\lambda_{N_2}) \} -  b_4(p'(\lambda_{N_1},\lambda_{N_2}))^{-1}\frac{d p'(\lambda_{N_1},\lambda_{N_2}) }{d \lambda_{N_1}} \\  \label{eq:grad1}
=&(1-p) \log \{ \frac{1-p'(\lambda_{N_1},\lambda_{N_2})}{1-q'(\lambda_{N_1},\lambda_{N_2})} \} +p \log \{\frac{p'(\lambda_{N_1},\lambda_{N_2}) }{ q'(\lambda_{N_1},\lambda_{N_2})} \} \\ \label{eq:grad2}
&+ \left(\frac{b_1}{1-q'(\lambda_{N_1},\lambda_{N_2})}-\frac{b_2}{q'(\lambda_{N_1},\lambda_{N_2})}\right)\frac{d q'(\lambda_{N_1},\lambda_{N_2}) }{d \lambda_{N_1}} \\ \label{eq:grad3}
&+ \left( \frac{b_3}{1-p'(\lambda_{N_1},\lambda_{N_2})}-\frac{b_4}{p'(\lambda_{N_1},\lambda_{N_2})}\right) \frac{d p'(\lambda_{N_1},\lambda_{N_2}) }{d \lambda_{N_1}}. 
\end{align}
Then, when $b_1 > b_2$, $q'(\lambda_{N_1},\lambda_{N_2})=\frac{b_2}{b_1+b_2}$, $1-q'(\lambda_{N_1},\lambda_{N_2})=\frac{b_1}{b_1+b_2}$ and we have 
\begin{align*}
\frac{b_1}{1-q'(\lambda_{N_1},\lambda_{N_2})}-\frac{b_2}{q'(\lambda_{N_1},\lambda_{N_2})} = \frac{b_1}{\frac{b_1}{b_1+b_2}}-\frac{b_2}{\frac{b_2}{b_1+b_2}} =0. 
\end{align*} On the other hand, when $b_1 \le b_2$, $\frac{d q'(\lambda_{N_1},\lambda_{N_2}) }{d \lambda_{N_1}}=0$. Performing a similar calculation for the last term, shows that lines~\eqref{eq:grad2} and~\eqref{eq:grad3} are 0.

Next, we analyze the terms in line~\eqref{eq:grad1}, 
\begin{align*}
&(1-p) \log \{ \frac{1-p'(\lambda_{N_1},\lambda_{N_2})}{1-q'(\lambda_{N_1},\lambda_{N_2})} \} +p \log \{\frac{p'(\lambda_{N_1},\lambda_{N_2}) }{ q'(\lambda_{N_1},\lambda_{N_2})} \}\\
=&\log \underbrace{\left( \frac{1-p'(\lambda_{N_1},\lambda_{N_2})}{1-q'(\lambda_{N_1},\lambda_{N_2})} \right)^{(1-p)}}_{<1}+ \log\underbrace{ \left( \frac{p'(\lambda_{N_1},\lambda_{N_2}) }{ q'(\lambda_{N_1},\lambda_{N_2})} \right)^{p}}_{>1}
\end{align*}
Notice that $p' \in [0.5,p] $, $q' \in [q,0.5] $, $1-p' \in [1-p,0.5]$ and $1-q' \in [0.5,1-q]$. To simplify our expression, lets drop the dependancy on  $\lambda_{N_1}$ and $\lambda_{N_2}$ for now. Since log is an increasing function, to show that the expression is positive, we require that
\begin{align*}
 \log \left( \frac{p' }{ q'} \right)^{p} \ge -\log \left( \frac{1-p'}{1-q'} \right)^{(1-p)}
& \iff 
   \log \left( \frac{p' }{ q'} \right)^{p} \ge \log \left( \frac{1-q'}{1-p'} \right)^{(1-p)} \\
   &\iff  \left( \frac{p' }{ q'} \right)^{p} \ge  \left( \frac{1-q'}{1-p'} \right)^{(1-p)} \\
   &\iff  (1-p')^{(1-p)} (p')^{p}\ge  (1-q')^{(1-p)}( q')^{p} 
\end{align*}
Since for $g(x) = (1-x)^{(1-p)}x^p$, $\forall x \in [0,0.5]$, $g(x) \in [0,0.5]$, $0.5 \ge (1-q')^{(1-p)}( q')^{p}$. \\
Then, consider $f(x) = \log \{  (1-x)^{(1-p)} (x)^{p} \} = (1-p) \log (1-x) + p \log x$, where 
\begin{align*}
f'(x) = \frac{-(1-p)}{1-x} + \frac{p}{x}. 
\end{align*}
For $p \ge x \ge 0.5$,
\begin{align*}
 \frac{p}{1-p} \ge  \frac{x}{1-p} \ge  \frac{x}{1-x} \implies  p (1-x)  \ge  x (1-p) \implies \frac{p}{x}  \ge  \frac{1-p} {1-x} \implies f'(x) \ge 0. 
\end{align*}
Then, we have
\begin{align*}
f(x) \ge f(0.5) = \log 0.5 \implies (1-x)^{(1-p)} x^{p} \ge 0.5 
\end{align*}
Hence, 
\begin{align*}
(1-p')^{(1-p)} (p')^{p}\ge 0.5 \ge (1-q')^{(1-p)}( q')^{p} 
\end{align*}
and 
\begin{align*}
\frac{d z(\lambda_{N_1},\lambda_{N_2})}{d\lambda_{N_1}} \ge 0. 
\end{align*}
We note that the above holds with strict equality for $p'> 0.5$ or $q' < 0.5$. \\

Repeating for $\lambda_{N_2}$, 
\begin{align} 
\nonumber
\frac{d z(\lambda_{N_1},\lambda_{N_2})}{d\lambda_{N_2}}=&-\frac{d b_1}{d \lambda_{N_2}} \log \{ 1-q'(\lambda_{N_1},\lambda_{N_2}) \} -  b_1\frac{d \log \{ 1-q'(\lambda_{N_1},\lambda_{N_2}) \}}{d \lambda_{N_2}} \\
\nonumber
&-\frac{d b_2}{d \lambda_{N_2}} \log \{ q'(\lambda_{N_1},\lambda_{N_2}) \} -  b_2\frac{d \log \{q'(\lambda_{N_1},\lambda_{N_2}) \}}{d \lambda_{N_2}} \\
\nonumber
&-\frac{d b_3}{d \lambda_{N_2}} \log \{ 1-p'(\lambda_{N_1},\lambda_{N_2}) \} -  b_3\frac{d \log \{ 1-p'(\lambda_{N_1},\lambda_{N_2}) \}}{d \lambda_{N_2}} \\
\nonumber
&-\frac{d b_4}{d \lambda_{N_2}} \log \{ p'(\lambda_{N_1},\lambda_{N_2}) \} -  b_4\frac{d \log \{p'(\lambda_{N_1},\lambda_{N_2}) \}}{d \lambda_{N_2}} \\
\nonumber
=&(1-q) \log \{ 1-q'(\lambda_{N_1},\lambda_{N_2}) \} +  b_1 (1-q'(\lambda_{N_1},\lambda_{N_2}))^{-1} \frac{d q'(\lambda_{N_1},\lambda_{N_2}) }{d \lambda_{N_2}} \\
\nonumber
&+q \log \{ q'(\lambda_{N_1},\lambda_{N_2}) \} -  b_2(q'(\lambda_{N_1},\lambda_{N_2}))^{-1}\frac{dq'(\lambda_{N_1},\lambda_{N_2}) }{d \lambda_{N_2}} \\
\nonumber
&-(1-q) \log \{ 1-p'(\lambda_{N_1},\lambda_{N_2}) \} +  b_3(1-p'(\lambda_{N_1},\lambda_{N_2}))^{-1} \frac{d p'(\lambda_{N_1},\lambda_{N_2}) }{d \lambda_{N_2}} \\
\nonumber
&-q \log \{ p'(\lambda_{N_1},\lambda_{N_2}) \} -  b_4(p'(\lambda_{N_1},\lambda_{N_2}))^{-1}\frac{d p'(\lambda_{N_1},\lambda_{N_2}) }{d \lambda_{N_2}} \\
\label{eq:grad4}
=&(1-q) \log \{ \frac{1-q'(\lambda_{N_1},\lambda_{N_2})}{1-p'(\lambda_{N_1},\lambda_{N_2})} \} +q \log \{\frac{q'(\lambda_{N_1},\lambda_{N_2}) }{ p'(\lambda_{N_1},\lambda_{N_2})} \} \\ \label{eq:grad5}
&+ \left(\frac{b_1}{1-q'(\lambda_{N_1},\lambda_{N_2})}-\frac{b_2}{q'(\lambda_{N_1},\lambda_{N_2})}\right)\frac{d q'(\lambda_{N_1},\lambda_{N_2}) }{d \lambda_{N_2}} \\ \label{eq:grad6}
&+ \left( \frac{b_3}{1-p'(\lambda_{N_1},\lambda_{N_2})}-\frac{b_4}{p'(\lambda_{N_1},\lambda_{N_2})}\right) \frac{d p'(\lambda_{N_1},\lambda_{N_2}) }{d \lambda_{N_2}}.
\end{align}
Similarly, lines~\eqref{eq:grad4} and~\eqref{eq:grad5} are 0 and we analyze line~\eqref{eq:grad6} as before, 
\begin{align*}
&(1-q) \log \{ \frac{1-q'(\lambda_{N_1},\lambda_{N_2})}{1-p'(\lambda_{N_1},\lambda_{N_2})} \} +q \log \{\frac{q'(\lambda_{N_1},\lambda_{N_2}) }{ p'(\lambda_{N_1},\lambda_{N_2})} \}\\
=&\log \underbrace{\left(  \frac{1-q'(\lambda_{N_1},\lambda_{N_2})}{1-p'(\lambda_{N_1},\lambda_{N_2})} \right)^{(1-q)}}_{>1}+ \log\underbrace{ \left( \frac{q'(\lambda_{N_1},\lambda_{N_2}) }{ p'(\lambda_{N_1},\lambda_{N_2})} \right)^{q}}_{<1} \ge 0 \\
& \iff \left( \frac{1-q'}{1-p'} \right)^{(1-q)} \ge \left( \frac{ p'}{q'} \right) ^{q} \iff (1-q')^{(1-q)} (q' )^q\ge (p')^q(1-p')^{(1-q)}.
\end{align*}
Then, for $g(x) = (1-x)^{(1-q)}x^q$, $\forall x \in [0.5,1]$, $g(x) \in [0,0.5]$, $0.5 \ge (p')^q(1-p')^{(1-q)}$. \\
Again, consider $f(x) = \log \{  (1-x)^{(1-q)} (x)^{q} \} = (1-q) \log (1-x) + q \log x$, where 
\begin{align*}
f'(x) = \frac{-(1-q)}{1-x} + \frac{q}{x}
\end{align*}
For $0.5 \ge x \ge q$,
\begin{align*}
 \frac{1-q}{q} \ge  \frac{1-x}{q} \ge  \frac{1-x}{x} \implies  x (1-q)  \ge  q (1-x) \implies \frac{1-q}{1-x}  \ge  \frac{q} {x} \implies f'(x) \le 0
\end{align*}
Then, we have
\begin{align*}
f(x) \ge f(0.5) = \log 0.5 \implies (1-x)^{(1-q)} (x)^{q} \ge 0.5 
\end{align*}
Hence,
\begin{align*}
(1-q')^{(1-q)} (q' )^q \ge 0.5 \ge (p')^q(1-p')^{(1-q)}
\end{align*} and
\begin{align*}
\frac{d z(\lambda_{N_1},\lambda_{N_2})}{d\lambda_{N_2}} \ge 0.
\end{align*}
Similarly, we note that the above holds with strict equality for $p'> 0.5$ or $q' < 0.5$. Hence, concluding the proof. 
\end{proof} 

\begin{lemma}
\label{lem:K+1K-1}
For a fixed $\lambda \in \P_{I}$, $C \in \mathcal{C}^K$,
\begin{align*} \min_{[C'] \in \Alt(C)/\sim} z(\lambda_{N_1}, \lambda_{N_2}) =
\begin{cases}
 \min_{[C'] \in (\Alt(C) \cap (\mathcal{C}^{K-1} \cup \mathcal{C}^{K+1}))/\sim} z(\lambda_{N_1}, \lambda_{N_2}), &M > K > 1 \\
   \min_{[C'] \in (\Alt(C) \cap \mathcal{C}^{K+1})/\sim} z(\lambda_{N_1}, \lambda_{N_2}), &K=1 \\
\min_{[C'] \in (\Alt(C) \cap \mathcal{C}^{K-1})/\sim} z(\lambda_{N_1}, \lambda_{N_2}), &K=M, 
\end{cases}
\end{align*}
where $\lambda_{N_1} = \sum_{(i,j)\in \Lambda_{N_1}([C'])} \lambda_{ij}$ and $\lambda_{N_2} = \sum_{(i,j)\in \Lambda_{N_2}([C'])} \lambda_{ij}$. 
\end{lemma}

\begin{proof} We will consider 3 cases of $[C'] \in \Alt(C)/\sim$ below and show that we can always find some $[\bar{C}] \in (\Alt(C) \cap (\mathcal{C}^{K-1} \cup \mathcal{C}^{K+1}))/\sim$ such that $\lambda_{\bar{N_1}}+\lambda_{\bar{N_2}}=\sum_{(i,j)\in \Lambda_{N_1}([\bar{C}])} \lambda_{ij}+\sum_{(i,j)\in \Lambda_{N_2}([\bar{C}])} \lambda_{ij} \le \lambda_{N_1}+\lambda_{N_2}$. \\

\underline{Case 1:} Consider $M\ge K\ge1$ and $[C'] \in (\Alt(C) \cap \mathcal{C}^{K} )/\sim$, an alternative clustering to $C$ with the same number of clusters, $K$. 
Then, there exists a $g_k$ such that for all $g'_{k'}$, $g_k \not\subset g'_{k'}$. (Because if for each $g_k$, there exists a $g'_{k'}$ such that $g_k \subset g'_{k'}$ and since $[C'] \in \Alt(C)/\sim$, there must exist a $g_j$ such that for all $g'_{k'}$, $g_j \ne g'_{k'}$, then there must be more than $M$ items which is a contradiction.) Then, we can choose a $g'_{k'}$ such that $g_k \cap g'_{k'} \ne \emptyset$ and $g_k \setminus (g_k \cap g'_{k'}) \ne \emptyset$. Let $N'_1 = \{ (i,j): i \in g_k \cap g'_{k'}, j \in g_k \setminus (g_k \cap g'_{k'}), j \ge i  \}\cup \{ (j,i): i \in g_k \cap g'_{k'}, j \in g_k \setminus (g_k \cap g'_{k'}), i \ge j  \} \subset \Lambda_{N_1}([C']) \cup \Lambda_{N_2}([C'])$. Now, we can find a $[\bar{C}] \in (\Alt(C) \cap \mathcal{C}^{K+1})/\sim$ where $\bar{g}_j = g_j, \forall j \ne k$, $\bar{g}_k = g_k \setminus (g_k \cap g'_{k'})$ and $\bar{g}_{K+1} = g_k \cap g'_{k'}$. Hence, we have $\Lambda_{N_1}([\bar{C}]) \cup \Lambda_{N_2}([\bar{C}]) = N'_1 \subset \Lambda_{N_1}([C']) \cup \Lambda_{N_2}([C'])$ and $\lambda_{\bar{N_1}} +\lambda_{\bar{N_2}} \le  \lambda_{N_1} +\lambda_{N_2}$. \\

\underline{Case 2:} Consider $M\ge K > 2$ and $[C'] \in (\Alt(C) \cap \mathcal{C}^{K-2} )/\sim$, an alternative clustering to $C$ with $K-2$ clusters. 
Then, there exists a $g_k$ such that $g_k \subset \bigcup_{k' \in J} g'_{k'}$ for some index set $J \subset [K-2]$ and for some $\tilde{k}' \in J$, $|g'_{\tilde{k}'}|>1$. (Because every $g_k$ is a subset of a sufficient number of $g'_{k'}$ and since $K-2 < M$, at least one $g'_{k'}$ must contain more than one item.) Now, we can find a $[\bar{C}] \in (\Alt(C) \cap \mathcal{C}^{K-1})/\sim$ where $\bar{g}_{K-1} = g_k$ and $\bar{g}_{k'} = g'_{k'} \setminus (g'_{k'} \cap g_k) \ne \emptyset$. Then, we also have $\Lambda_{N_1}([\bar{C}]) \cup \Lambda_{N_2}([\bar{C}]) \subset \Lambda_{N_1}([C']) \cup \Lambda_{N_2}([C'])$ and $\lambda_{\bar{N_1}} +\lambda_{\bar{N_2}} \le  \lambda_{N_1} +\lambda_{N_2}$. \\

\underline{Case 3:} Consider $M-1 > K \ge 1$ and $[C'] \in (\Alt(C) \cap \mathcal{C}^{K+2} )/\sim$, an alternative clustering to $C$ with $K+2$ clusters. 
Then, there exists a $g_k$ such that for all $g'_{k'}$, $g_k \not\subset g'_{k'}$. (Because if for all $g_k$, there exists a $g'_{k'}$ such that $g_k \subset g'_{k'}$, then there must be more than $M$ items or a maximum of $K$ clusters, which is a contradiction.) Now, similar to case 1, we can find a $[\bar{C}] \in \Alt(C)/\sim$ with $K+1$ clusters such that $\lambda_{\bar{N_1}} +\lambda_{\bar{N_2}} \le  \lambda_{N_1} +\lambda_{N_2}$. \\

Similar results can be obtained for $[C'] \in \Alt(C)/\sim$ with more than $K+2$ and less than $K-2$ clusters. Then, as $\lambda_{\bar{N_1}} +\lambda_{\bar{N_2}} \le  \lambda_{N_1} +\lambda_{N_2} \implies z(\lambda_{\bar{N_1}},\lambda_{\bar{N_2}}) \le z(\lambda_{N_1},\lambda_{N_2})$, from these 3 cases, if $M > K > 1$, we only need to consider $[C'] \in (\Alt(C) \cap (\mathcal{C}^{K-1} \cup \mathcal{C}^{K+1}))/\sim$ to solve the minimum. Further, for $K=M$, we only need to consider $[C'] \in (\Alt(C) \cap \mathcal{C}^{M-1} )/\sim$ and for $K=1$, we only need to consider $[C'] \in (\Alt(C) \cap \mathcal{C}^{2} )/\sim$. 
\end{proof}

\subsection{Proof of Proposition \ref{prop:uniform}} \label{app:uniform} 



\begin{proof}
We first consider the case where all $M$ items are in a single cluster. We know from Lemmas~\ref{lem:pqstar},~\ref{lem:K+1K-1} and Case 1 of the proof of Theorem \ref{thm:inner_inf} that 
\begin{align}
\label{eq:K=1}
\max_{\lambda \in \P_{I}} \inf_{C' \in \Alt_1(C)} h(\lambda, C')  = \max_{\lambda \in \P_{I}} \min_{[C'] \in (\Alt_1(C) \cap \mathcal{C}^{2})/\sim} z(\lambda_{N_1}, \lambda_{N_2}),  
\end{align}
where $\lambda_{N_1} = \sum_{(i,j)\in \Lambda_{N_1}([C'])} \lambda_{ij}$ and $\lambda_{N_2} = \sum_{(i,j)\in \Lambda_{N_2}([C'])} \lambda_{ij}$. Since $\Lambda_{N_2}([C'])=\emptyset$, for all $[C'] \in (\Alt_1(C) \cap \mathcal{C}^{2})/\sim$ as $N_q=\emptyset$, we only need to consider $\lambda_{N_1}$ and by Lemma \ref{lem:inc_z}, $z(\lambda_{N_1}, \lambda_{N_2})$ is increasing on $\lambda_{N_1}$. Hence, we simply need to solve~\eqref{eq:K=1} over $\lambda_{N_1}$ instead of $z$. 

To keep the presentation clear, we drop the equivalence class notation for $C'$ and $\Alt_1(C) \cap \mathcal{C}^{2}$ and represent each $C'$ by the size of and items in its smaller cluster. Then, we can partition $\Alt_1(C) \cap \mathcal{C}^2$ into sets with $i$ items in the smaller cluster, for $i \in [\floor*{M/2}]$. For $C' \in \Alt_1(C) \cap \mathcal{C}^2$, relabel $g_1 = \argmin\{|g'_1|, |g'_2|\}$ and $g_2 = \argmax\{|g'_1|, |g'_2|\}$. We reuse the notation of $f$ to map each $C'$ to its respective $g_1$, i.e. $f: \Alt_1(C) \cap \mathcal{C}^2 \rightarrow \bigcup_{C'\in \Alt_1(C) \cap \mathcal{C}^2} \{g_1\}$ and we only need to consider $\lambda_{N_1}$ for all such $g_1$. Note that $g_1$ depends on $C'$ since $g_1$ represents the cluster assignments in $C'$. Let $\mathcal{G}_1=\bigcup_{C'\in \Alt_1(C) \cap \mathcal{C}^2} \{g_1\}$. Then, we can partition $\mathcal{G}_1$ into sets of $g_1$ that contain exactly $i$ items, $\mathcal{G}_1=\bigcup_{i \in [\floor*{M/2}]} \{g_1 \in \mathcal{G}_1: |g_1|=i\}\coloneqq \bigcup_{i \in [\floor*{M/2}]} \mathcal{G}_1^i$. 

For $g_1 \in \mathcal{G}_1^i$, $|\Lambda_{N_1}(f^{-1}(g_1))|=i(M-i)$.\footnote{Picture each matrix $C'$ that represents $g_1$ as a permutation of the simplest block matrix where items $1,\cdots,i$ are in $g_1$}  From the bijective map $f$, we also have a bijection, with a slight abuse of notation, $\lambda_{N_1}: \mathcal{G}_1 \rightarrow \{\sum_{(i,j)\in \Lambda_{N_1}(f^{-1}(g_1))} \lambda_{ij}\}$. If $\lambda$ is uniform, then $\min_{g_1 \in \mathcal{G}_1} \lambda_{N_1}(g_1)=(M-1)/{\binom{M}{2}} \in \lambda_{N_1}(\mathcal{G}_1^1)$, since $\min_{i \in  [\floor*{M/2}]} i(M-i)=M-1$. Here, $\lambda_{N_1}$ maps $g_1 \mapsto \sum_{(i,j)\in \Lambda_{N_1}(f^{-1}(g_1))} \lambda_{ij}$. 

Suppose for a contradiction that the optimal $\lambda$ is \emph{not} uniform. Then, $\min_{g_1 \in \mathcal{G}_1} \lambda_{N_1}(g_1)>(M-1)/{\binom{M}{2}}$ and let $\epsilon' = \min_{g_1 \in \mathcal{G}_1} \lambda_{N_1}(g_1)-(M-1)/{\binom{M}{2}}>0$. Further, let $i \ge 2$ be the index such that the minimum occurs at some $g_1 \in \mathcal{G}_1^{i}$. Without loss of generality, let $\lambda_{ij} = \frac{M-1}{i(M-i){\binom{M}{2}}}+\frac{\epsilon'}{i(M-i)}$ for all $(i,j) \in \Lambda_{N_1}(f^{-1}(g_1))$. \\

\textit{In the following 2 paragraphs, for ease of notation, we will use the convention where $(i,j)$ indexes item $i\in g_1$ and item $j\in g_2$ and not worry that $j > i$. We can do this since for each $(i,j)$ where $j < i$, $(j,i)\in I$ and we may assume to refer to that index instead.}\\

Consider $g_1 \supset g'_1 = \{i_1\} \in \mathcal{G}_1^1$ and $\lambda_{N_1}(g'_1)$. Then, for all $(i_1,j)$ for $j\in g_2 \subset g'_2$, $(i_1,j) \in \Lambda_{N_1}(f^{-1}(g'_1))$ and we can partition $\Lambda_{N_1}(f^{-1}(g_1)) = \{(i_1,j): j \in g_2\} \cup \{(i_1,j): j \in [M]\setminus (\{i_1\}\cup g_2)\}$. So, $\lambda_{N_1}(g'_1) = \frac{M-1}{i{\binom{M}{2}}} + \frac{\epsilon'}{i} + \sum_{\{(i_1,j): j \in [M]\setminus (\{i_1\}\cup g_2)\}} \lambda_{i_1,j}$ and $\min_{g_1 \in \mathcal{G}_1} \lambda_{N_1}(g_1)=\lambda_{N_1}(g_1)=\frac{M-1}{{\binom{M}{2}}}+\epsilon'$. Since $\lambda_{N_1}(g_1) \le \lambda_{N_1}(g'_1)$, $$\sum_{\{(i_1,j): j \in [M]\setminus (\{i_1\}\cup g_2)\}} \lambda_{(i_1,j)} \ge \frac{(i-1)(M-1)}{i{\binom{M}{2}}}+\frac{(i-1)\epsilon'}{i}.$$ Doing this for $i_k=i_1,i_2,\cdots,i_i \in g_1$, we have $$\sum_{\{(i_k,j): j \in [M]\setminus (\{i_k\}\cup g_2)\}} \lambda_{i_k,j} \ge \frac{(i-1)(M-1)}{i{\binom{M}{2}}}+\frac{(i-1)\epsilon'}{i}, $$ where we have counted each pair of items in $g_1$ in such a sum twice. 

Next, consider $g_2 \supset g''_1 = \{j_1\}\in \mathcal{G}_1^1$. Similarly, we can partition $N''_1 = \{(i,j_1): i \in g_1\} \cup \{(i,j_1): i \in [M]\setminus (\{j_1\}\cup g_1)\}$ and $\lambda_{N_1}(g''_1)=\frac{M-1}{(M-i){\binom{M}{2}}} + \frac{\epsilon'}{M-i} + \sum_{(i,j_1): i \in [M]\setminus (\{j_1\}\cup g_1)} \lambda_{(i,j_1)}$. Since $\lambda_{N_1}(g_1) \le \lambda_{N_1}(g''_1)$, $$\sum_{(i,j_1): i \in [M]\setminus (\{j_1\}\cup g_1)} \lambda_{i,j_1} \ge \frac{(M-1)(M-i-1)}{(M-i){\binom{M}{2}}} + \frac{(M-i-1)\epsilon'}{M-i}.$$ Again, doing this for $j_k=j_1,j_2,\cdots,j_{M-i} \in g_2$, we have $$\sum_{(i,j_k): i \in [M]\setminus (\{j_k\}\cup g_1)} \lambda_{i,j_k} \ge \frac{(M-1)(M-i-1)}{(M-i){\binom{M}{2}}} + \frac{(M-i-1)\epsilon'}{M-i}.$$ That is, we will count every pair of items in $g_2$ twice in such a sum. 

Then, we notice that $N_1 \coloneqq \Lambda_{N_1}(f^{-1}(g_1)) \subset g_1 \times g_2$, $\tilde{N}'_k\coloneqq \{(i_k,j): j \in [M]\setminus (\{i_k\}\cup g_2)\} \subset g_1 \times g_1$ and $\tilde{N}''_k\coloneqq \{(i,j_k): i \in [M]\setminus (\{j_k\}\cup g_1)\} \subset g_2 \times g_2$. So, they are disjoint and $$\sum_{(i,j)\in N_1 \cup \bigcup_{k\in [i]} \tilde{N}'_k \cup \bigcup_{k\in [M-i]} \tilde{N}''_k} \lambda_{(i,j)} \ge \frac{M-1}{{\binom{M}{2}}}+\epsilon' + \frac{(i-1)(M-1)}{{\binom{M}{2}}}+(i-1)\epsilon' + \frac{(M-1)(M-i-1)}{{\binom{M}{2}}} + (M-i-1)\epsilon',$$ where $\bigcup_{k\in [i]} \tilde{N}'_k =\{(i,j)\in g_1 \times g_1: j \ne i \}$ and $\bigcup_{k\in [M-i]} \tilde{N}''_k = \{(i,j)\in g_2 \times g_2: j \ne i \}$. Then, let $\tilde{g}_1^2 = \{(i,j)\in g_1 \times g_1: j > i \}$ and $\tilde{g}_2^2 = \{(i,j)\in g_2 \times g_2: j > i \}$ such that we have 
\begin{align*}
    &\sum_{(i,j)\in N_1 \cup \tilde{g}_1^2 \cup \tilde{g}_2^2} \lambda_{ij} \\
    &\quad \ge \frac{M-1}{{\binom{M}{2}}}+\epsilon' + \frac{(i-1)(M-1)}{2{\binom{M}{2}}}+\frac{(i-1)\epsilon'}{2} + \frac{(M-1)(M-i-1)}{2{\binom{M}{2}}} + \frac{(M-i-1)\epsilon'}{2} \\
    &\quad = \frac{M-1}{{\binom{M}{2}}}+\epsilon' + \frac{(M-2)(M-1)}{2{\binom{M}{2}}} + \frac{(M-2)\epsilon'}{2}\\
    &\quad = \frac{(M-1)M/2}{{\binom{M}{2}}}+\epsilon' + \frac{(M-2)\epsilon'}{2} \\
    &\quad = 1+\epsilon' + \frac{(M-2)\epsilon'}{2}> 1.
\end{align*} Since $\sum_{(i,j)\in N_1 \cup \tilde{g}_1^2 \cup \tilde{g}_2^2} \lambda_{i,j} = \sum_{(i,j)\in I} \lambda_{i,j}=1$, we have a contradiction. \\

Then, suppose $i=1$ such that $\min_{g_1 \in \mathcal{G}_1} \lambda_{N_1}(g_1)=\lambda_{N_1}(g_1)$ for some $g_1 \in \mathcal{G}_1^1$. Without loss of generality, let $g_1=\{1\}$ and $\min_{g_1 \in \mathcal{G}_1} \lambda_{N_1}(g_1) = \frac{M-1}{{\binom{M}{2}}} + \epsilon'$ for some $\epsilon'>0$. We notice that $|\mathcal{G}_1^1|=M$ since $\mathcal{G}_1^1=\{\{k\}: k\in [M] \}$. For $N_1^k \coloneqq \Lambda_{N_1}(f^{-1}(\{k\}))$, since every pair $(i,j)\in I$ occurs exactly twice, in $N_1^i$ and $N_1^j$, we must have $\sum_{k \in [M]} \lambda_{N_1}(\{k\}) = 2$. By the definition of $\min_{g_1 \in \mathcal{G}_1} \lambda_{N_1}(g_1)$, $\lambda_{N_1}(\{k\}) \ge \lambda_{N_1}(\{1\}) = \frac{M-1}{{\binom{M}{2}}} + \epsilon'$, for all $k \in [M]\setminus \{1\}$. Hence, $\sum_{k \in [M]} \lambda_{N_1}(\{k\}) \ge \frac{M(M-1)}{{\binom{M}{2}}} + M\epsilon' > 2$, leading to a contradiction. \\

Next, consider the case where all $M$ items are in separate clusters, i.e. $K=M$. Then, similar to before, we know that 
\begin{align}
\label{eq:K=M}
\max_{\lambda \in \P_{I}} \inf_{C' \in \Alt_1(C)} h(\lambda, C')  = \max_{\lambda \in \P_{I}} \min_{[C'] \in (\Alt_1(C) \cap \mathcal{C}^{M-1})/\sim} z(\lambda_{N_1}, \lambda_{N_2}),   
\end{align} and we only need to consider minimizing over $\lambda_{N_2}$. It is easy to see that for all $[C'] \in (\Alt_1(C) \cap \mathcal{C}^{M-1})/\sim$, $\Lambda_{N_1}([C'])=\emptyset$ and $\Lambda_{N_2}([C'])=\{(i,j) \}$, where items $i$ and $j$ are combined to form a cluster and $\bigcup_{[C'] \in (\Alt_1(C) \cap \mathcal{C}^{M-1})/\sim} \Lambda_{N_2}([C']) = I$. If $\lambda^*$ is not the uniform distribution, then there exists an $(i,j) \in I$ such that $\lambda_{ij} < 1/{\binom{M}{2}}$ and $\min_{[C'] \in (\Alt_1(C) \cap \mathcal{C}^{M-1})/\sim} \lambda_{N_2} < 1/{\binom{M}{2}}$. Since the uniform distribution achieves $\min_{[C'] \in (\Alt_1(C) \cap \mathcal{C}^{M-1})/\sim} \lambda_{N_2} = 1/{\binom{M}{2}}$, it should be the optimal distribution. 
\end{proof} 
\section{Proofs of Results in Section \ref{sec:achievability}}

\subsection{Proof of Proposition \ref{prop:mix_optsample}} \label{app:prf_mix_optsample}

\begin{proof}
 By the strong law of large numbers, the event $\mathcal{E} = \{\omega \in \Omega: \lim_{t \rightarrow \infty} \hat{C}_t(\omega) = C\}$ has probability $1$. For $\omega \in \mathcal{E}$, $\forall \mu > 0$, $\exists t_\mu$ such that $\forall t \ge t_\mu$, $|\hat{c}_t(\omega)_{ij} - c_{ij}| < \mu$, $\forall (i,j) \in I$. Let $\mu = \min\{p-0.5,0.5-q\}$, then $\forall t \ge t_\mu$, if $c_{ij}=p$, we have $\hat{c}_t(\omega)_{ij}-0.5 > -\mu + c_{ij}-0.5 =   -\mu + p-0.5 \ge 0$, and if $c_{ij}=q$, we have $\hat{c}_t(\omega)_{ij}-0.5 < \mu + c_{ij}-0.5 =  \mu + q-0.5 \le 0$. Then, since in Algorithm \ref{alg:proj}, we set $(\hat{c}_{t_=})_{ij} = \I \{(\hat{c}_t)_{ij} \ge 0.5 \}$ and $(\hat{c}_{t_{\ne}})_{ij}=\I \{(\hat{c}_t)_{ij} < 0.5 \}$, when solving~\eqref{eq:proj_hatC}, we recover $c_{t_=}$ as $c_=$ and $c_{t_{\ne}}$ as $c_{\ne}$ exactly and $[C_t(\omega)]=[C]$. Furthermore, 
    \begin{align*}
        |p_t-p| &= \left|\frac{\sum_{(i,j)\in I} \I \{(c_{t_=})_{ij}=1\} (\hat{c}_t)_{ij} \cdot N_{ij}(t)}{\sum_{(i,j)\in I} \I \{(c_{t_=})_{ij}=1\} N_{ij}(t)}-p \right| \\
       &=  \left|\frac{\sum_{(i,j)\in I} \I \{(c_{t_=})_{ij}=1\} (\hat{c}_t)_{ij} \cdot N_{ij}(t)}{\sum_{(i,j)\in I} \I \{(c_{t_=})_{ij}=1\} N_{ij}(t)}-\frac{\sum_{(i,j)\in I} \I \{(c_{t_=})_{ij}=1\} N_{ij}(t)}{\sum_{(i,j)\in I} \I \{(c_{t_=})_{ij}=1\} N_{ij}(t)}p\right| \\
       &= \left|\frac{\sum_{(i,j)\in I} \I \{(c_{t_=})_{ij}=1\} N_{ij}(t) \cdot ((\hat{c}_t)_{ij}-p)}{\sum_{(i,j)\in I} \I \{(c_{t_=})_{ij}=1\} N_{ij}(t)}\right| \\*
       &< \mu.
\end{align*}
   Performing a similar calculation for $q_t$ yields, $|q_t-q|<\mu$ as well. 
   
Then, for $\omega \in \mathcal{E}$ and $\mu > 0$, since by Lemma \ref{lem:sigma_cont}, $\lambda^*(\sigma;C)=\bar{g}_\sigma(p,q)$ is continuous in $x=(p,q)$ on $D$ if $[C_t(\omega)]=[C]$, there exists a $\min\{p-0.5,0.5-q\}>\eta>0$ such that for all $t\ge t_\eta$ and $y_t=(p_t,q_t)$, we have $\|y_t-x\| < \eta$ and hence $\|\lambda^*(\sigma;C_t(\omega))-\lambda^*(\sigma;C))\| < \mu/3({M \choose 2}-1)(1-\epsilon)$. This leads to
\begin{align*}
    \|\lambda_\epsilon^*(\sigma;C_t(\omega))-\lambda_\epsilon^*(\sigma;C))\| &= \left\|(1-\epsilon)\lambda^*(\sigma;C_t(\omega))+\epsilon \frac{1}{|I|}-(1-\epsilon)\lambda^*(\sigma;C))-\epsilon \frac{1}{|I|} \right\| \\ 
    &=(1-\epsilon) \|\lambda^*(\sigma;C_t(\omega))-\lambda^*(\sigma;C))\| \\
    &< \frac{\mu}{3 \big({M \choose 2}-1 \big)}. 
\end{align*}
For ease of notation, we drop the $\omega$ term as it is implicitly implied by $C_t$. By  \cite[Lemma 17]{garivier2016optimal}, there exists a $t_\mu \ge t_\eta$ such that for all $t \ge t_\mu$, $\|\frac{N(t)}{t}-\lambda_\epsilon^*(\sigma;C))\| < \mu$. 

\end{proof}


\begin{lemma}
\label{lem:sigma_cont}
    For fixed $[C]$, let $C \in [C]$ and a domain $D=(0.5,1] \times [0,0.5)$. Define $\bar{g}: D \rightarrow \mathscr{P}(\mathcal{P}_{I})$ whereby
    \begin{align*}
\bar{g}_\sigma(p,q)\coloneqq \argmax_{\lambda \in \mathcal{P}_{I}} \min_{C' \in \min(C)/\sim} h(\lambda, C';p,q)-\sigma R(\lambda),
    \end{align*}
where $R(\lambda) = \frac{1}{2}\|\lambda\|^2$ is strictly convex and $\sigma > 0$.
Then, $\bar{g}_\sigma$ is continuous on $D$. 
\end{lemma}
\begin{proof}
From the definition of $h(\lambda,C')$, 
\begin{align*}
   h(\lambda,C') &=
     d\left(p,\min  \left\{0.5,  \frac{\lambda_{N_1}}{\lambda_{N_1}+\lambda_{N_2}'}p + \frac{\lambda_{N_2}'}{\lambda_{N_1}+\lambda_{N_2}'}q \right\}\right)  \lambda_{N_1} \\
   &\qquad + d\left(q,\min \left\{0.5,  \frac{\lambda_{N_1}}{\lambda_{N_1}+\lambda_{N_2}'}p + \frac{\lambda_{N_2}'}{\lambda_{N_1}+\lambda_{N_2}'}q \right\}\right) \lambda'_{N_2} \\
   &\qquad +  d\left(q, \max \left\{0.5,  \frac{\lambda_{N_1}'}{\lambda_{N_1}'+\lambda_{N_2}}p + \frac{\lambda_{N_2}}{\lambda_{N_1}'+\lambda_{N_2}}q \right\}\right)\lambda_{N_2} \\
   &\qquad+  d\left(p, \max \left\{0.5,  \frac{\lambda_{N_1}'}{\lambda_{N_1}'+\lambda_{N_2}}p + \frac{\lambda_{N_2}}{\lambda_{N_1}'+\lambda_{N_2}}q \right\}\right)\lambda'_{N_1},
\end{align*}
where $\lambda_{N_1} = \sum_{(i,j)\in \Lambda_{N_1}(C')} \lambda_{ij}$, $\lambda_{N_2} = \sum_{(i,j)\in \Lambda_{N_2}(C')} \lambda_{ij}$, $\lambda'_{N_1} = \sum_{(i,j)\in N_p} \lambda_{ij}-\lambda_{N_1}$ and $\lambda'_{N_2} = \sum_{(i,j)\in N_q} \lambda_{ij}-\lambda_{N_2}$, it should be clear that $(\lambda, p,q)\rightarrow h(\lambda,C';p,q)$ is continuous on $\P_{I}\times D$ for fixed $C'$. Then, since we are taking a finite minimum operation on $h(\lambda,C';p,q)$ over $\min(C)/\sim$ and $R(\lambda)$ is continuous, $\min_{C' \in \min(C)/\sim} h(\lambda, C';p,q)- \sigma R(\lambda)$ is similarly continuous. Furthermore, as $\P_{I}$ does not depend on $p$ or $q$ and is compact, we can apply Berge's Maximum Theorem \cite{berge1997topological} to show that $\bar{g}_\sigma(p,q)$ is upper-hemicontinuous and hence, continuous since it is single-valued.
\end{proof}

\color{black}
\subsection{Proof of Proposition \ref{prop:stopping_rule_delta}} \label{app:stopping_rule_delta}

\begin{proof} We simply bound the error probabilities as follows,
\begin{align*} 
\Pr (\tau_\delta < \infty, C_{\tau_\delta} \not\sim C) &\le \Pr (\exists t \in \mathbb{N}: C_t \not\sim C, Z(t) > \beta(t,\delta)) \\
&= \Pr (\exists t \in \mathbb{N}: C \in \Alt(C_t), Z(t) > \beta(t,\delta)) \\
&\le \Pr \left(\exists t \in \mathbb{N}: \sum_{(i,j)\in I} N_{ij}(t)d((\hat{c}_t)_{ij},c_{ij}) > \beta(t,\delta)\right) \\
&\le \delta,
\end{align*}
where the last inequality is due to  \cite[Theorem 7]{kaufmann2021mixture}, with $x=\ln(1/\delta)$. 
\end{proof}

\subsection{Proof of Theorem \ref{thm:upper_bound}}
\label{app:prf_stopping_rule}

\begin{proof}
First, by Proposition~\ref{prop:mix_optsample},
under the D-tracking rule, for all $(i,j) \in I$, $N_{ij}(t) \ge (\sqrt{t} - {\binom{M}{2}}/2)_+ -1$ and by  \cite[Lemma 17]{garivier2016optimal}, $\forall \mu > 0$, $\exists t'_\mu \ge t_\mu$ such that
\begin{align*}
\sup_{t \ge t_\mu} \max_{(i,j)\in I} |\lambda_\epsilon^*(\sigma;C_t)_{ij}-\lambda_\epsilon^*(\sigma;C)_{ij}| < \mu \implies
\sup_{t \ge t'_\mu} \max_{(i,j)\in I} \left|\frac{N_{ij}(t)}{t}-\lambda_\epsilon^*(\sigma;C)_{ij} \right| \le 3\left({\binom{M}{2}} -1 \right)\mu. 
\end{align*}

Let $0 < \mu < \min\{p-0.5,0.5-q\}$, $R(\lambda):= \|\lambda\|^2/2$, and define the event 
\begin{align*}
\Ev_T(\mu)\coloneqq \bigcap_{t=h(T)}^\infty \left\{\omega \in \Omega: \max_{(i,j)\in I} |\hat{c}_t(\omega)_{ij}-c_{ij}|<\mu\right\}
\end{align*}
and functions 
\begin{align*}
g(\tilde{\lambda}, \tilde{C}) &\coloneqq \inf_{C' \in \Alt(C)} \sum_{(i,j)\in I} \tilde{\lambda}_{ij} d(\tilde{c}_{ij},c'_{ij})-\sigma R(\tilde{\lambda}) \\
\hat{M}(t) &\coloneqq \inf_{C' \in \Alt(C)} \sum_{(i,j)\in I} \frac{N_{ij}(t)}{t}d((\hat{c}_t)_{ij},c'_{ij})-\sigma R\left(\frac{N_{ij}(t)}{t} \right) = g\left(\frac{N(t)}{t}, \hat{C}_t\right) \\
C_\mu(C) &\coloneqq \inf_{\substack{\tilde{C}: \max_{(i,j) \in I} |\tilde{c}_{ij}-c_{ij}|<\mu \\ \tilde{\lambda}: \max_{(i,j)\in I} |\tilde{\lambda}_{ij}-\lambda_\epsilon^*(\sigma;C)_{ij}|<3 \left({\binom{M}{2}}-1 \right)\mu}} g(\tilde{\lambda}, \tilde{C}), 
\end{align*}
for some increasing function $h(t)$. We notice here that since the $\Alt(C)$ in all of the functions defined \textbf{does not} depend on the input, it is not necessary that $\tilde{C} \in \mathcal{C}$ and we can use $\hat{C}_t$ directly. Further, $g$ is a continuous function. 

On the event $\Ev_T(\mu)$, we have that $\forall t \ge h(T)$, 
\begin{enumerate}
\item[(i)] $[C_t] = [C]$
\item[(ii)]  $\Alt(C_t) = \Alt(C)$
\item[(iii)]  $Z(t) = \inf_{C' \in \Alt(C)} \sum_{(i,j)\in I} N_{ij}(t)d((\hat{c}_t)_{ij},c'_{ij})\le \beta(t,\delta) \implies t\hat{M}(t)\le \beta(t,\delta) $
\item[(iv)]  $\exists t'_\mu \ge h(T)$ such that $\forall t \ge t'_\mu$, $\sup_{t \ge t'_\mu} \max_{(i,j)\in I} |\frac{N_{ij}(t)}{t}-\lambda_\epsilon^*(\sigma;C)_{ij}|<3 ({\binom{M}{2}}-1)\mu$ 
\item[(v)]  $\forall t \ge t'_\mu$, $\hat{M}(t) \ge C_\mu(C)$.
\end{enumerate}
Part (iii) holds since $Z(t)=t\hat{M}(t)+t\sigma R(\frac{N(t)}{t})$ and $t\sigma R(\frac{N(t)}{t})$ is positive. Part (v) follows from part (iv) and the definition of sets of $\tilde{\lambda}$ and $\tilde{C}$ that we minimize over in $C_\mu(C)$. 

Now, we can choose $T_1$ and $h(T) > T$ such that $h(T_1) \ge {\binom{M}{2}}^2$ and $T$ such that $\sqrt{T} \ge t'_\mu$. Then, on $\Ev_{T}(\mu)$, 
\begin{align*}
\min(\tau_\delta, T) &\le \sqrt{T} + \sum_{t=\sqrt{T}}^{T} \I \{ \tau_\delta > t \} \overset{\text{(iii)}}{\le} \sqrt{T} + \sum_{t=\sqrt{T}}^{T} \I \{ t\hat{M}(t) \le \beta(t,\delta) \} \\
& \overset{\text{(v)}}{\le} \sqrt{T} + \sum_{t=\sqrt{T}}^{T} \I \{ tC_\mu(C)\le \beta(T,\delta) \} \le \sqrt{T} + \frac{\beta(T,\delta)}{C_\mu(C)}. 
\end{align*}
Let $T_2(\delta) = \inf \{T \in \mathbb{N}: \sqrt{T} + \frac{\beta(T,\delta)}{C_\mu(C)} \le T \}$ and $T_3 = \max \{T_2(\delta), (t'_\mu)^2\}$, then $\tau_\delta \le T_3$. ($T-\sqrt{T}$ is an increasing function and $T_2(\delta)$ is the smallest $T$ such that on $\Ev_{T}(\mu)$, $T\ge \tau_\delta$) Hence, we have that $\forall T \ge T_3$, $\Ev_{T}(\mu) \subset \{ \tau_\delta \le T \}$ and $\Pr( \Ev^c_{T}(\mu)) \ge \Pr( \tau_\delta > T)$. 

Then, so far we have
\begin{align*}
\E_C[\tau_\delta] &= \sum_{t=1}^{T_3-1} \Pr(\tau_\delta > t) + \sum_{t=T_3}^{\infty} \Pr(\tau_\delta > t) \\ &\le T_3 +\sum_{t=T_3}^{\infty} \Pr( \Ev^c_{t}(\mu)) \\
&\le T_2(\delta) + (t'_\mu)^2 + \sum_{t=T_3}^{\infty} \Pr( \Ev^c_{t}(\mu))
\end{align*}

Now, we will tackle each term in turn, starting with providing an upper bound on $T_2(\delta)$.
Let $\xi > 0$ and $C(\xi) = \inf \{T \in \mathbb{N}: T - \sqrt{T} \ge T/(1+\xi) \}$. 
Then, using the upper bound on $\beta(t,\delta)$ from Lemma \ref{lemma:upperbound_beta}, we have
\begin{align*}
T_2(\delta) &\le C(\xi) + \inf \bigg\{ T \in \mathbb{N}: \frac{\beta(T,\delta)}{C_\mu(C)} \le \frac{T}{1+\xi}\bigg\} \\
&\le C + C(\xi) + \inf\bigg\{ T \in \mathbb{N}: \frac{\ln(\frac{DT^{\alpha}}{\delta^{\gamma}})}{C_\mu(C)} \le \frac{T}{1+\xi}\bigg\} \\
&= C + C(\xi) + \inf\bigg\{ T \in \mathbb{N}: \frac{\ln(\frac{D^{1/\alpha}T}{\delta^{\gamma/\alpha}})}{C_\mu(C)} \le \frac{T}{\alpha(1+\xi)}\bigg\} 
\end{align*}
Since 
\begin{align*}
h\left(\frac{TC_\mu(C)}{\alpha(1+\xi)}\right) = \frac{TC_\mu(C)}{\alpha(1+\xi)} - \ln \left(\frac{TC_\mu(C)}{\alpha(1+\xi)} \right),
\end{align*}
for $T$ such that $\frac{\ln\left(\frac{D^{1/\alpha}T}{\delta^{\gamma/\alpha}}\right)}{C_\mu(C)} \le \frac{T}{\alpha(1+\xi)}$, we have
\begin{align*}
\ln\left(\frac{D^{1/\alpha}T}{\delta^{\gamma/\alpha}}\right) - \ln \left(\frac{C_\mu(C)}{\alpha(1+\xi)} \right)&\le \frac{TC_\mu(C)}{\alpha(1+\xi)} -  \ln \left(\frac{C_\mu(C)}{\alpha(1+\xi)} \right) \\
\iff \ln\left(\frac{D^{1/\alpha}}{\delta^{\gamma/\alpha}}\right) - \ln \left(\frac{C_\mu(C)}{\alpha(1+\xi)} \right)&\le \frac{TC_\mu(C)}{\alpha(1+\xi)} -  \ln \left(\frac{TC_\mu(C)}{\alpha(1+\xi)} \right)\\
\iff \ln\left(\frac{\alpha(1+\xi) D^{1/\alpha}}{C_\mu(C)\delta^{\gamma/\alpha}}\right) &\le h\left(\frac{TC_\mu(C)}{\alpha(1+\xi)}\right). 
\end{align*}
Then, as $h$ is an increasing function, 
\begin{align*}
 h^{-1}\left(\ln(\frac{\alpha(1+\xi) D^{1/\alpha}}{C_\mu(C)\delta^{\gamma/\alpha}})\right) &\le \frac{TC_\mu(C)}{\alpha(1+\xi)} \quad \iff \quad \frac{\alpha(1+\xi)}{C_\mu(C)}h^{-1}\left(\ln\Big(\frac{\alpha(1+\xi) D^{1/\alpha}}{C_\mu(C)\delta^{\gamma/\alpha}}\Big)\right) \le T. 
\end{align*}
Substituting this inequality and using Proposition 8 in \cite{kaufmann2021mixture}, we have 
\begin{align*}
T_2(\delta) \le C + C(\xi) &+ \inf \left\{ T \in \mathbb{N}: \frac{\ln(\frac{D^{1/\alpha}T}{\delta^{\gamma/\alpha}})}{C_\mu(C)} \le \frac{T}{\alpha(1+\xi)}\right\} \\
= C + C(\xi) &+ \inf\left\{ T \in \mathbb{N}:  \frac{\alpha(1+\xi)}{C_\mu(C)}h^{-1}\left(\ln \bigg(\frac{\alpha(1+\xi) D^{1/\alpha}}{C_\mu(C)\delta^{\gamma/\alpha}}\bigg)\right) \le T \right\} \\
\le C + C(\xi)   &+\frac{\alpha(1+\xi)}{C_\mu(C)}h^{-1}\left(\ln\bigg(\frac{\alpha(1+\xi) D^{1/\alpha}}{C_\mu(C)\delta^{\gamma/\alpha}} \bigg)\right) + 1 \\
= C + C(\xi)   &+\frac{\alpha(1+\xi)}{C_\mu(C)}h^{-1}\left(\kappa + \ln\left(\frac{1}{\delta^{\gamma/\alpha}}\right)\right) + 1 \\
\le C + C(\xi) &+1 \\
&\hspace{-1.8cm}+\frac{\alpha(1+\xi)}{C_\mu(C)} \left[\kappa + \ln\left(\frac{1}{\delta^{\gamma/\alpha}} \right) + \ln \left( \kappa + \ln\left(\frac{1}{\delta^{\gamma/\alpha}} \right) + \sqrt{2\left(\kappa + \ln(\frac{1}{\delta^{\gamma/\alpha}}) - 1\right)} \right) \right],
\end{align*} where $\kappa \coloneqq \ln(\frac{\alpha(1+\xi) D^{1/\alpha}}{C_\mu(C)})$.

Next, we upper bound $\sum_{t=T_3}^{\infty} \Pr( \Ev^c_{t}(\mu))$. 
For $T \ge T_3$, $h(T) \ge {\binom{M}{2}}^2$ and by the Chernoff–Hoeffding theorem \cite{409cf137-dbb5-3eb1-8cfe-0743c3dc925f}, 
\begin{align*}
\Pr( \Ev^c_T(\mu)) &= \Pr \Bigg( \bigcup_{t=h(T)}^\infty \Big\{ \max_{(i,j)\in I} |(\hat{c}_t)_{ij}-c_{ij}|\ge\mu\Big\} \Bigg)\\
&\le \sum_{t=h(T)}^\infty \Pr \left(  \max_{(i,j)\in I} |(\hat{c}_t)_{ij}-c_{ij}|\ge\mu \right) \\
&\le  \sum_{t=h(T)}^\infty  \sum_{(i,j) \in I}\Pr ( |(\hat{c}_t)_{ij}-c_{ij}|\ge\mu )\\
&= \sum_{t=h(T)}^\infty  \sum_{(i,j) \in I}\Pr ( (\hat{c}_t)_{ij}-c_{ij}\ge\mu) + \Pr ( c_{ij}-(\hat{c}_t)_{ij}\ge\mu) \\
&\le \sum_{t=h(T)}^\infty  \sum_{(i,j) \in I} e^{-d(c_{ij}+\mu, c_{ij})N_{ij}(t)} + e^{-d(c_{ij}-\mu, c_{ij})N_{ij}(t)}
\end{align*}

By the D-tracking rule and the fact that $t \ge {\binom{M}{2}}^2$, $N_{ij}(t) \ge (\sqrt{t}-{\binom{M}{2}}/2)_+-1 \ge \sqrt{t} - {\binom{M}{2}}$ for $M \ge 3$. Let $d_\mu = \min \{ d(p+\mu,p),d(p-\mu,p),d(q+\mu,q),d(q-\mu,q) \}$. Then,
\begin{align*}
\Pr ( \Ev^c_T(\mu)) &\le \sum_{t=h(T)}^\infty  \sum_{(i,j) \in I} e^{-d(c_{ij}+\mu, c_{ij})N_{ij}(t)} + e^{-d(c_{ij}-\mu, c_{ij})N_{ij}(t)} \\
&\le \sum_{t=h(T)}^\infty 2 {\binom{M}{2}} e^{-d_\mu (\sqrt{t} - {\binom{M}{2}})}\\
&=2 {\binom{M}{2}} e^{d_\mu{\binom{M}{2}}} \sum_{t=h(T)}^\infty  e^{-d_\mu \sqrt{t}} \\
&\le 2 {\binom{M}{2}} e^{d_\mu{\binom{M}{2}}} \int_{t=h(T)-1}^\infty  e^{-d_\mu \sqrt{t}} dt \\
&= 4 {\binom{M}{2}} e^{d_\mu{\binom{M}{2}}}  \int_{t=\sqrt{h(T)-1}}^\infty  t e^{-d_\mu t} dt \\
&= 4 {\binom{M}{2}} e^{d_\mu{\binom{M}{2}}}  \left(\frac{\sqrt{h(T)-1}}{d_\mu}+ \frac{1}{d^2_\mu} \right)e^{-d_\mu \sqrt{h(T)-1} }
\end{align*}
Then, since $h(t)\ge t$ is an increasing function, we can consider $h(t)$ to be a subsequence of $t$ and we have
\begin{align*}
\sum_{t=T_3}^{\infty} \Pr( \Ev^c_{t}(\mu)) &\le 4 {\binom{M}{2}} e^{d_\mu{\binom{M}{2}}}  \sum_{t=T_3}^{\infty} \bigg(\frac{\sqrt{h(t)-1}}{d_\mu}+ \frac{1}{d^2_\mu}\bigg)e^{-d_\mu \sqrt{h(t)-1} } \\
&\le \frac{4 {\binom{M}{2}} e^{d_\mu{\binom{M}{2}}} }{d_\mu} \sum_{t=T_3}^{\infty} \left(\sqrt{t-1}+ \frac{1}{d_\mu} \right)e^{-d_\mu \sqrt{t-1} }. 
\end{align*}
Since the infinite sum on the right hand side is convergent, so is $\sum_{t=T_3}^{\infty} \Pr( \Ev^c_{t}(\mu))$. 

Finally, we have 
\begin{align*}
\limsup_{\delta \rightarrow 0}\frac{\E_C[\tau_\delta]}{\ln(1/\delta)} &\le \limsup_{\delta \rightarrow 0}\frac{T_2(\delta)}{\ln(1/\delta)} \\
&\le \frac{\alpha(1+\xi)}{C_\mu(C)}\limsup_{\delta \rightarrow 0}\frac{ \kappa + \frac{\gamma}{\alpha}\ln(1/\delta) + \ln \left( \kappa + \frac{\gamma}{\alpha}\ln(1/\delta) + \sqrt{2\left(\kappa + \frac{\gamma}{\alpha} \ln(1/\delta) - 1\right)} \right) }{\ln(1/\delta)} \\
&= \frac{\gamma(1+\xi)}{C_\mu(C)}
\end{align*}
By letting $\mu$ and $\xi$ go to zero, using the continuity of $g$, we obtain our upper bound
\begin{align*}
\limsup_{\delta \rightarrow 0}\frac{\E_C[\tau_\delta]}{\ln(1/\delta)} \le  \gamma D^*_\epsilon(\sigma;C)^{-1}, 
\end{align*} where $\gamma > 1$ is a constant that can be made arbitrarily close to 1. 

Since $\inf_{C' \in \Alt(C)} \sum_{(i,j)\in I}{\lambda_\epsilon^*(\sigma;C)}_{ij} d(c_{ij},c'_{ij}) >0$ for all $\sigma>0$, there must exist a small enough $\sigma$ such that $D_\epsilon^*(\sigma;C)>0$.
\end{proof}

We define the functions required to fully describe $\beta(t,\delta)$ below,
\begin{align*}
h(u) &= u - \ln (u), \quad u \ge 1 \\ 
\text{for some $z \in [1,e]$, }\tilde{h}_z(x) &= \begin{cases}
e^{1/h^{-1}(x)}h^{-1}(x) \text{, if $x \ge h(1/\ln(z))$} \\
z(x-\ln \ln (z)) \text{, otherwise}
\end{cases}\\ 
\zeta(s) &= \sum_{n=1}^\infty n^{-s}\\
\mathcal{C}_{\mathrm{exp}}(x) &= 2 \tilde{h}_{3/2}\left( \frac{h^{-1}(1+x) + \ln(2 \zeta (2))}{2}\right) \\
\beta(t,\delta) &= 3 \sum_{(i,j)\in I} \ln (1+ \ln N_{ij}(t)) + {\binom{M}{2}} \mathcal{C}_{\mathrm{exp}}\bigg(\frac{\ln(1/\delta)}{{\binom{M}{2}}} \bigg). 
\end{align*}
It is clear that $\beta(t,\delta)$ is increasing in $t$. 

\begin{lemma}
\label{lemma:upperbound_beta}
There exists constants $\alpha, C, D > 0$ and $\gamma > 1$ such that $\forall t \ge C$, $\delta \in (0,1)$, 
\begin{align*}
\beta(t,\delta) \le \ln\bigg(\frac{Dt^{\alpha}}{\delta^{\gamma}}\bigg).  
\end{align*}
\end{lemma}

\begin{proof}
We start by noticing a few properties of $h$ and $h^{-1}$. Firstly, that since $h$ is defined on $u \ge 1$, $h^{-1}(u) \ge 1$. Then, by Lemma 29 in \cite{kaufmann2021mixture}, we have, for $x \ge 0$, $z \in [1,e]$, 
\begin{align*}
\tilde{h}_z(x) = \min_{y \in [1,z]} y(x-\ln \ln(y))
\end{align*}
and by the variational representation of $h^{-1}(x)$, for $\gamma_2 > 1$
\begin{align*}
h^{-1}(x) = \inf_{z \ge 1} z\left( x -1+\ln(\frac{z}{z-1})\right) \le \gamma_2\left(x - 1 + \ln (\frac{\gamma_2}{\gamma_2-1})\right).
\end{align*}
Combining these results, for $\gamma_1 \in (1,3/2]$, we have
\begin{align*}
\mathcal{C}_{\mathrm{exp}}\bigg(\frac{\ln(1/\delta)}{{\binom{M}{2}}}\bigg) &= 2 \tilde{h}_{3/2}\left( \frac{h^{-1}(1+\frac{\ln(1/\delta)}{{\binom{M}{2}}}) + \ln(2 \zeta (2))}{2}\right) \\
&= 2\min_{y \in [1,3/2]} y\left(\frac{h^{-1}(1+\frac{\ln(1/\delta)}{{\binom{M}{2}}}) + \ln(2 \zeta (2))}{2}-\ln \ln(y)\right) \\
&\le 2 \cdot \gamma_1\left(\frac{1}{2}h^{-1}(1+\frac{\ln(1/\delta)}{{\binom{M}{2}}}) +\frac{1}{2} \ln(2 \zeta (2)) - \ln \ln(\gamma_1)\right) \\
&=\gamma_1\left(h^{-1}(1+\frac{\ln(1/\delta)}{{\binom{M}{2}}}) + \ln(2 \zeta (2)) - 2 \ln \ln(\gamma_1)\right) \\
&\le \gamma_1\left(  \gamma_2 \left(\frac{\ln(1/\delta)}{{\binom{M}{2}}} + \ln (\frac{\gamma_2}{\gamma_2-1}) \right) + \ln(2 \zeta (2))- 2 \ln \ln(\gamma_1) \right)\\
&=  \frac{\gamma_1 \gamma_2 \ln(1/\delta)}{{\binom{M}{2}}} + \gamma_1 \gamma_2 \ln (\frac{\gamma_2}{\gamma_2-1}) +\gamma_1 \ln(2 \zeta (2))- 2\gamma_1 \ln \ln(\gamma_1)
\end{align*}
and 
\begin{align*}
{\binom{M}{2}} \mathcal{C}_{\mathrm{exp}}\bigg(\frac{\ln(1/\delta)}{{\binom{M}{2}}}\bigg)  &\le \gamma_1 \gamma_2\ln(1/\delta) + \underbrace{{\binom{M}{2}}\left(\gamma_1 \gamma_2 \ln (\frac{\gamma_2}{\gamma_2-1}) +\gamma_1 \ln(2 \zeta (2))- 2\gamma_1 \ln \ln(\gamma_1) \right)}_{\coloneqq a>0}. 
\end{align*}
Then, $\forall \delta \in (0,1)$, let $D_1 > e^a$ such that 
\begin{align*}
{\binom{M}{2}} \mathcal{C}_{\mathrm{exp}}(\frac{\ln(1/\delta)}{{\binom{M}{2}}})  &\le \gamma_1 \gamma_2\ln(1/\delta) + \ln(D_1)
\end{align*}
Next, we look to finding an upper bound on $3 \sum_{(i,j)\in I} \ln (1+ \ln N_{ij}(t))$ and notice that $\ln(1+\ln x)$ is concave on $x \ge e^{-2}$. Then, for $\min_{(i,j)\in I}  N_{ij}(t) \ge 1$, by Jensen's inequality \cite{grimmett2001} we have,
\begin{align*}
\sum_{(i,j)\in I} \frac{1}{|I|} \ln (1+ \ln N_{ij}(t)) &\le \ln (1+ \ln \sum_{(i,j)\in I} \frac{1}{|I|}  N_{ij}(t)) \\
&= \ln (1+ \ln t / |I|)\\
&\le 1+ \ln t / |I|. 
\end{align*}
Then, since
\begin{align*}
3 \sum_{(i,j)\in I} \ln (1+ \ln N_{ij}(t)) \le 3|I| (1+ \ln t / |I|) \le 3|I|  - 3|I| \ln(I)+ 3|I|\ln(t),
\end{align*}
we can choose $D_2 > e^{3|I|  - 3|I| \ln(I)}$ and $C$ such that $\forall t \ge C, \min_{(i,j)\in I}  N_{ij}(t) \ge 1$ to obtain 
\begin{align*}
3 \sum_{(i,j)\in I} \ln (1+ \ln N_{ij}(t)) \le \ln(D_2)+ 3|I|\ln(t). 
\end{align*}
Putting these inequalities together, we have $\forall \delta \in (0,1),\forall t \ge C$ there is a $ D = D_1 D_2 > 0$ such that for $ \gamma_1 \in (1,3/2], \gamma_2 > 1$, 
\begin{align*}
\beta(t,\delta) &= 3 \sum_{(i,j)\in I} \ln (1+ \ln N_{ij}(t)) + {\binom{M}{2}} \mathcal{C}_{\mathrm{exp}}\bigg(\frac{\ln(1/\delta)}{{\binom{M}{2}}} \bigg) \\
&\le \ln(D_2)+ 3|I|\ln(t) + \gamma_1 \gamma_2\ln(1/\delta) + \ln(D_1) \\
&= \ln\bigg(\frac{Dt^{3{\binom{M}{2}}}}{\delta^{\gamma_1 \gamma_2}}\bigg) \coloneqq \ln\bigg(\frac{Dt^{\alpha}}{\delta^{\gamma}}\bigg)
\end{align*}
\end{proof}
\color{black}

\subsection{Proof of Theorem \ref{thm:mix_ALTupper_bound}} \label{app:mix_ALTupper_bound}

\begin{proof}
First, under the D-tracking rule, for all $(i,j) \in I$, $N_{ij}(t) \ge (\sqrt{t} - {\binom{M}{2}}/2)_+ -1$ and by  \cite[Lemma 17]{garivier2016optimal}, $\forall \mu > 0$, $\exists t'_\mu \ge t_\mu$ such that
\begin{align*}
\sup_{t \ge t_\mu} \max_{(i,j)\in I} \left|\lambda_\epsilon^*(\sigma;C_t)_{ij}-\lambda_\epsilon^*(\sigma;C)_{ij} \right| < \mu \implies
\sup_{t \ge t'_\mu} \max_{(i,j)\in I} \left|\frac{N_{ij}(t)}{t}-\lambda_\epsilon^*(\sigma;C)_{ij}\right| \le 3\left({\binom{M}{2}} -1\right)\mu. 
\end{align*}

Let $0 < \mu < \min\{p-0.5,0.5-q\}$ and $R(\lambda)= \frac{1}{2}\|\lambda\|^2$. Define the event 
\begin{align*}
\Ev_T(\mu)\coloneqq \bigcap_{t=h(T)}^\infty \left\{\omega \in \Omega: \max_{(i,j)\in I} |\hat{c}_t(\omega)_{ij}-c_{ij}|<\mu\right\}
\end{align*}
and functions 
\begin{align*}
g(\tilde{\lambda}, \tilde{C}) &\coloneqq \Big[ \min_{(i',j')\in I} \tilde{\lambda}_{i'j'} \Big]\cdot \inf_{C' \in \Alt(C)} \sum_{(i,j)\in I}  d(\tilde{c}_{ij},c'_{ij}) -\sigma R(\tilde{\lambda}) \\
\hat{M}(t) &\coloneqq \Big[\min_{(i',j')\in I} \frac{N_{i'j'}(t)}{t}\Big] \cdot \inf_{C' \in \Alt(C)} \sum_{(i,j)\in I} d((\hat{c}_t)_{ij},c'_{ij})-\sigma R \left(\frac{N(t)}{t} \right) = g\left(\frac{N(t)}{t}, \hat{C}_t\right) \\
C_\mu(C) &\coloneqq \inf_{\substack{\tilde{C}: \max_{(i,j) \in I} |\tilde{c}_{ij}-c_{ij}|<\mu \\ \tilde{\lambda}: \max_{(i,j) \in I} |\tilde{\lambda}_{ij}-\lambda_\epsilon^*(\sigma;C)_{ij}|<3 ({\binom{M}{2}}-1)\mu}} g(\tilde{\lambda}, \tilde{C}), 
\end{align*}
for some increasing function $h(t)$. We observe that $g$ is still a continuous function. 

On the event $\Ev_T(\mu)$, we have that $\forall t \ge h(T)$, 
\begin{enumerate}
\item $[C_t] = [C]$
\item $\Alt(C_t) = \Alt(C)$
\item $Z(t) = \big[\min_{(i',j')\in I} N_{i'j'}(t)\big] \cdot\inf_{C' \in \Alt(C)}  \sum_{(i,j)\in I} d((\hat{c}_t)_{ij},c'_{ij}) \le \beta(t,\delta) \implies t\hat{M}(t)\le \beta(t,\delta) $
\item $\exists t'_\mu \ge h(T)$ such that $\forall t \ge t'_\mu$, 
$ \max_{(i,j)\in I} |\frac{N_{ij}(t)}{t}-\lambda_\epsilon^*(\sigma;C)_{ij}| \le 3({\binom{M}{2}} -1)\mu$
\item $\forall t \ge t'_\mu$, $\hat{M}(t) \ge C_\mu(C)$.
\end{enumerate}
3. holds since $Z(t)=t\hat{M}(t)+t\sigma R(\frac{N(t)}{t})$ and $t\sigma R(\frac{N(t)}{t})$ is positive. 
5. follows from 4. and the definition of sets of $\tilde{\lambda}$ and $\tilde{C}$ that we minimize over in $C_\mu(C)$. 

Now, we can choose $T_1$ and $h(T) > T$ such that $h(T_1) \ge {\binom{M}{2}}^2$ and $T$ such that $\sqrt{T} \ge t'_\mu$. Then, on $\Ev_{T}(\mu)$, 
\begin{align*}
\min(\tilde{\tau}_\delta, T) &\le \sqrt{T} + \sum_{t=\sqrt{T}}^{T} \I \{ \tilde{\tau}_\delta > t \} \overset{3.}{\le} \sqrt{T} + \sum_{t=\sqrt{T}}^{T} \I \{ t\hat{M}(t) \le \beta(t,\delta) \} \\
& \overset{5.}{\le} \sqrt{T} + \sum_{t=\sqrt{T}}^{T} \I \{ tC_\mu(C)\le \beta(T,\delta) \} \le \sqrt{T} + \frac{\beta(T,\delta)}{C_\mu(C)}. 
\end{align*}
Let $T_2(\delta) = \inf \{T \in \mathbb{N}: \sqrt{T} + \frac{\beta(T,\delta)}{C_\mu(C)} \le T \}$ and $T_3 = \max \{T_2(\delta), (t_\mu)^2\}$, then $\hat{\tau}_\delta \le T_3$. ($T-\sqrt{T}$ is an increasing function) Hence, we have that $\forall T \ge T_3$, $\Ev_{T}(\mu) \subset \{ \hat{\tau}_\delta \le T \}$ and $\Pr( \Ev^c_{T}(\mu)) \ge \Pr( \hat{\tau}_\delta > T)$. 

Then, so far we have
\begin{align*}
\E_C[\tilde{\tau}_\delta] &= \sum_{t=1}^{T_3-1} \Pr(\tilde{\tau}_\delta > t) + \sum_{t=T_3}^{\infty} \Pr(\tilde{\tau}_\delta > t) \\ &\le T_3 +\sum_{t=T_3}^{\infty} \Pr( \Ev^c_{t}(\mu)) \\
&\le T_2(\delta) + (t'_\mu)^2 + \sum_{t=T_3}^{\infty} \Pr( \Ev^c_{t}(\mu)). 
\end{align*}
We can follow the same steps as Appendix \ref{app:prf_stopping_rule}, to upper bound $T_2(\delta)$ and $\sum_{t=T_3}^{\infty} \Pr( \Ev^c_{t}(\mu))$.

Then, we have 
\begin{align*}
\limsup_{\delta \rightarrow 0}\frac{\E_C[\tilde{\tau}_\delta]}{\ln(1/\delta)} \le \frac{\gamma(1+\xi)}{C_\mu(C)}.
\end{align*}
Since
\begin{align*}
    C_\mu(C) &= \inf_{\substack{\tilde{C}: \max_{(i,j) \in I} |\tilde{c}_{ij}-c_{ij}|<\mu \\ \tilde{\lambda}: \max_{(i,j) \in I} |\tilde{\lambda}_{ij}-\lambda^*(\sigma;C)_{ij}|<3 ({\binom{M}{2}}-1)\mu}} g(\tilde{\lambda}, \tilde{C})
\end{align*} and $g$ is a continuous function, as $\mu \rightarrow 0$, 
\begin{align*}
    C_\mu(C) \rightarrow 
    g(\lambda_\epsilon^*(\sigma;C),C)&=\big[ \min_{(i',j')\in I} \lambda_\epsilon^*(\sigma;C)_{i'j'} \big]\cdot \inf_{C' \in \Alt(C)} \sum_{(i,j)\in I}  d(c_{ij},c'_{ij}) -\sigma R(\lambda_\epsilon^*(\sigma;C)) \\
   & \coloneqq \tilde{D}_\epsilon^*(\sigma;C). 
\end{align*}

Since $\big[\min_{(i',j')\in I} {\lambda_\epsilon^*(\sigma;C)}_{i'j'}\big] \cdot \big[\inf_{C' \in \Alt(C)} \sum_{(i,j)\in I} d(c_{ij},c'_{ij})\big] \ge \frac{\epsilon}{|I|}\big[\inf_{C' \in \Alt(C)} \sum_{(i,j)\in I} d(c_{ij},c'_{ij})\big]>0$ for all $\sigma>0$, there must exist a small enough $\sigma$ such that $\tilde{D}_\epsilon^*(\sigma;C)>0$.

\end{proof}

\section{On the Computationally Feasible Stopping Rule} \label{app:computationally_feas}
Below we derive a closed-form expression for $\widehat{Z}(t)$, thus making it a computationally tractable and  practical alternative to the GLR-based statistic $Z(t)$ (whose computation is infeasible due to lack of a closed-form expression). 
We show that the search over $\Alt(C_t)$ in \eqref{eq:ALTstopping_stat} can be reduced to a search over finitely many equivalence classes. Additionally, we show that instead of evaluating~\eqref{eq:ALTstopping_stat} in full, it suffices to simply evaluate the maximum and minimum of $\hat{p}_n$, achieving a significant gain in efficiency. 

\begin{proposition}
\label{prop:ALTstopping_rule_simplify} Let $(i_t^* , j_t^*) \in \arg\min_{(i ,j )\in I} N_{i j } (t)$, and for a fixed alternate instance $C'\in\Alt(C_t)$, define $N'_p=N'_p(C') =\{(i,j)\in I: c'_i=c'_j\}$ and let $n = |N_p'|$ and $n'=\binom{M}{2}-n$. Additionally, let 
\begin{align*}
\hat{p}_n = \frac{1}{n}\sum_{(i,j)\in N'_p} (\hat{c}_t)_{ij},      \qquad 
\hat{p}^{\mathrm{max}}_n =\max_{\substack{ C'\in \Alt(C_t)/\sim: \\ |N_p'|=n} }\hat{p}_n, \qquad
\hat{p}^{\mathrm{min}}_n = \min_{\substack{C'\in \Alt(C_t)/\sim: \\|N'_p|=n, \hat{C}_{\mathrm{tot}}-n\hat{p}_n<0.5n'}} \hat{p}_n,
\end{align*}
  where $\hat{C}_{\mathrm{tot}} =\sum_{(i,j)\in I} (\hat{c}_t)_{ij} $.   Let $s_n(p ) = n H(p) + n' H(\frac{\hat{C}_{\mathrm{tot}}-np}{n'})$ where $H(p)=-p\log p -(1-p)\log (1-p)$ is the binary entropy of $\mathrm{Ber}(p)$. Then, the computationally feasible stopping statistic $\widehat{Z}(t)$ can be simplified as 
\begin{align*}
   \widehat{Z}(t) = 
    N_{i_t^* j_t^*} (t) \cdot \begin{cases}
    \displaystyle\min_{n \in [{\binom{M}{2}}]} s_n(\hat{p}^{\mathrm{max}}_n)- \sum_{(i,j) \in I} H((\hat{c}_t)_{ij}), & \text{for } \hat{C}_{\mathrm{tot}} > \frac{{\binom{M}{2}}}{2} \\
    \displaystyle\min_{n \in [{\binom{M}{2}}]} \min\{ s_n(\hat{p}^{\mathrm{max}}_n), s_n(\hat{p}^{\mathrm{min}}_n) \}- \sum_{(i,j) \in I} H((\hat{c}_t)_{ij}), & \text{for } \hat{C}_{\mathrm{tot}} \le \frac{{\binom{M}{2}}}{2}
    \end{cases}.
\end{align*}
\end{proposition}

\begin{proof}
As usual, we first define sets to group the terms in our objective function in the alternate stopping rule
\begin{align*}
N'_p =N'_p(C')= \{ (i,j)\in I: c'_{i} = c'_j\}; \hspace{0.5cm} N'_q = N'_q(C')= \{ (i,j)\in I: c'_{i} \ne c'_j\}.
\end{align*}
Then, we simplify the objective function by expanding the KL divergence
\begin{align}
\label{eq:alt_stopping_inf}
\nonumber
\widehat{Z}(t) &= \inf_{C' \in \Alt(C_t)}\Big[\min_{(i',j')\in I} N_{i'j'} (t) \Big] \sum_{(i,j)\in I} d((\hat{c}_t)_{ij},c'_{ij})\\ \nonumber
&=N_{i^*_t,j^*_t}(t) \inf_{C' \in \Alt(C_t)} \sum_{(i,j)\in I} d((\hat{c}_t)_{ij},c'_{ij}) \\ \nonumber
&= N_{i^*_t,j^*_t}(t)\inf_{C' \in \Alt(C_t)} \sum_{(i,j)\in N'_p} d((\hat{c}_t)_{ij},p')+ \sum_{(i,j)\in N'_q} d((\hat{c}_t)_{ij},q') \\ \nonumber
&
\begin{multlined}[0.8\linewidth]
 =N_{i^*_t,j^*_t}(t) \inf_{C' \in \Alt(C_t)} \sum_{(i,j)\in N'_p} (\hat{c}_t)_{ij} \log \frac{(\hat{c}_t)_{ij}}{p'} +(1-(\hat{c}_t)_{ij}) \log \frac{1-(\hat{c}_t)_{ij}}{1-p'} \\ \nonumber+\sum_{(i,j)\in N'_q}  (\hat{c}_t)_{ij} \log \frac{(\hat{c}_t)_{ij}}{q'} +(1-(\hat{c}_t)_{ij}) \log \frac{1-(\hat{c}_t)_{ij}}{1-q'} 
\end{multlined} \\ \nonumber
& \begin{multlined}[0.9\linewidth]= N_{i^*_t,j^*_t}(t) \inf_{C' \in \Alt(C_t)} - \sum_{(i,j)\in N'_p} (\hat{c}_t)_{ij} \log p' -(1-(\hat{c}_t)_{ij}) \log (1-p') \\ \nonumber
-\sum_{(i,j)\in N'_q} (\hat{c}_t)_{ij} \log q' -(1-(\hat{c}_t)_{ij}) \log (1-q') - \sum_{(i,j) \in I} H((\hat{c}_t)_{ij})
\end{multlined} \\
&=  N_{i^*_t,j^*_t}(t) \inf_{C' \in \Alt(C_t)} s(C') - \sum_{(i,j) \in I} H((\hat{c}_t)_{ij}),
\end{align}
where $p'$ and $q'$ are the probabilities associated with instance $C'$ and $s(C')=- \sum_{(i,j)\in N'_p} (\hat{c}_t)_{ij} \log p' -(1-(\hat{c}_t)_{ij}) \log (1-p')-\sum_{(i,j)\in N'_q} (\hat{c}_t)_{ij} \log q' -(1-(\hat{c}_t)_{ij}) \log (1-q')$. 
For ease of notation, we drop the primes for $N'_p$, $N'_q$, $p'$ and $q'$. It should make sense that within an equivalence class, $[C'] \in \Alt(C_t)/\sim$, the optimal $p$ and $q$ would be the empirical average of $(\hat{c}_t)_{ij}$ in $N_p$ and $N_q$ respectively. For completeness, we solve for them below. Differentiating $s$ with respect to $p$, we obtain
\begin{align*}
    \frac{ds}{dp} =|N_p|\left( -\frac{\sum_{(i,j)\in N_p} (\hat{c}_t)_{ij}}{|N_p|}\frac{1}{p} + \left(1-\frac{\sum_{(i,j)\in N_p} (\hat{c}_t)_{ij}}{|N_p|}\right)\frac{1}{1-p}\right). 
\end{align*}

\underline{Case 1:} $\frac{\sum_{(i,j)\in N_p} (\hat{c}_t)_{ij}}{|N_p|} < 0.5 < p$\\
Then, we have
\begin{align*}
    \frac{\sum_{(i,j)\in N_p} (\hat{c}_t)_{ij}}{|N_p|}\frac{1}{p} < 1 ,\quad\quad 
    \left(1-\frac{\sum_{(i,j)\in N_p} (\hat{c}_t)_{ij}}{|N_p|}\right)\frac{1}{1-p} > 1
\end{align*}
and $\frac{ds}{dp} > 0$ for $0.5< p < 1$. Hence, $p=0.5$. \\

\underline{Case 2:} $\frac{\sum_{(i,j)\in N_p} (\hat{c}_t)_{ij}}{|N_p|} > 0.5$\\
If $\frac{\sum_{(i,j)\in N_p} (\hat{c}_t)_{ij}}{|N_p|} > p > 0.5$, then, 
\begin{align*}
    \frac{\sum_{(i,j)\in N_p} (\hat{c}_t)_{ij}}{|N_p|}\frac{1}{p} > 1 ,\quad\quad 
    \left(1-\frac{\sum_{(i,j)\in N_p} (\hat{c}_t)_{ij}}{|N_p|}\right)\frac{1}{1-p} < 1
\end{align*}
and $\frac{ds}{dp} < 0$. \\
If $p > \frac{\sum_{(i,j)\in N_p} (\hat{c}_t)_{ij}}{|N_p|} > 0.5$, then, we have case 1 and 
$\frac{ds}{dp} > 0$. Hence, $p = \frac{\sum_{(i,j)\in N_p} (\hat{c}_t)_{ij}}{|N_p|}$. \\
Combining both cases, we have $p = \max \{0.5, \frac{\sum_{(i,j)\in N_p} (\hat{c}_t)_{ij}}{|N_p|}\}$. 

Similarly for $q$,
\begin{align*}
    \frac{ds}{dq} =|N_q|\left( -\frac{\sum_{(i,j)\in N_q} (\hat{c}_t)_{ij}}{|N_q|}\frac{1}{q} + \left(1-\frac{\sum_{(i,j)\in N_q} (\hat{c}_t)_{ij}}{|N_q|}\right)\frac{1}{1-q}\right). 
\end{align*}

\underline{Case 1:} $\frac{\sum_{(i,j)\in N_q} (\hat{c}_t)_{ij}}{|N_q|} > 0.5 > q$ \\
Then, we have
\begin{align*}
    \frac{\sum_{(i,j)\in N_q} (\hat{c}_t)_{ij}}{|N_q|}\frac{1}{q} > 1 ,\quad\quad 
    \left(1-\frac{\sum_{(i,j)\in N_q} (\hat{c}_t)_{ij}}{|N_q|}\right)\frac{1}{1-q} < 1
\end{align*}
and $\frac{ds}{dq} < 0$ for $0.5> q > 0$. Hence, $q=0.5$. \\

\underline{Case 2:} $\frac{\sum_{(i,j)\in N_q} (\hat{c}_t)_{ij}}{|N_q|} < 0.5$\\
If $\frac{\sum_{(i,j)\in N_q} (\hat{c}_t)_{ij}}{|N_q|} < q < 0.5$, then, 
\begin{align*}
    \frac{\sum_{(i,j)\in N_q} (\hat{c}_t)_{ij}}{|N_q|}\frac{1}{q} < 1 ,\quad\quad 
    \left(1-\frac{\sum_{(i,j)\in N_q} (\hat{c}_t)_{ij}}{|N_q|}\right)\frac{1}{1-q} > 1
\end{align*}
and $\frac{ds}{dq} > 0$. \\
If $ 0.5> \frac{\sum_{(i,j)\in N_q} (\hat{c}_t)_{ij}}{|N_q|} > q$, then, we have case 1 and 
$\frac{ds}{dq} < 0$. Hence, $q = \frac{\sum_{(i,j)\in N_q} (\hat{c}_t)_{ij}}{|N_q|}$. \\
Combining both cases once more, we have $q = \min \{0.5, \frac{\sum_{(i,j)\in N_q} (\hat{c}_t)_{ij}}{|N_q|}\}$.

Next, we simplify $s$ after substituting in our optimal $p$ and $q$. Fix $|N_p|$ and let $n\coloneqq |N_p|$, $n'\coloneqq |N_q| = {\binom{M}{2}}-n$, $\hat{p} \coloneqq\sum_{(i,j)\in N_p} (\hat{c}_t)_{ij} $, $\hat{C}_{\mathrm{tot}} \coloneqq\sum_{(i,j)\in I} (\hat{c}_t)_{ij} $, $\hat{q} \coloneqq \sum_{(i,j)\in N_q} (\hat{c}_t)_{ij} = \hat{C}_{\mathrm{tot}} - \hat{p}$. Then, if $\frac{\sum_{(i,j)\in N_p} (\hat{c}_t)_{ij}}{|N_p|}\ge 0.5$ and $\frac{\sum_{(i,j)\in N_q} (\hat{c}_t)_{ij}}{|N_q|}\le0.5$, $p = \frac{\hat{p}}{n}$ and $q = \frac{\hat{C}_{\mathrm{tot}}-\hat{p}}{n'}$. Since we fixed $n$, the only dependent variable is $\hat{p}$ and we can write $s$ as
\begin{align*}
-s_n(\hat{p}) &= \sum_{(i,j)\in N_p} (\hat{c}_t)_{ij} \log p +(1-(\hat{c}_t)_{ij}) \log (1-p) +\sum_{(i,j)\in N_q} (\hat{c}_t)_{ij} \log q +(1-(\hat{c}_t)_{ij}) \log (1-q) \\
&= \hat{p} \log \frac{\hat{p}}{n} + (n - \hat{p}) \log \left(1- \frac{\hat{p}}{n}\right) + (\hat{C}_{\mathrm{tot}}-\hat{p})\log \frac{\hat{C}_{\mathrm{tot}}-\hat{p}}{n'} \\
&\quad + (n'-\hat{C}_{\mathrm{tot}}+\hat{p})\log\left(1- \frac{\hat{C}_{\mathrm{tot}}-\hat{p}}{n'}\right) \\
&= \hat{p} \log \frac{\hat{p}}{n} + (n - \hat{p}) \log \left(\frac{n-\hat{p}}{n}\right) + (\hat{C}_{\mathrm{tot}}-\hat{p})\log \frac{\hat{C}_{\mathrm{tot}}-\hat{p}}{n'} \\
&\quad + (n'-\hat{C}_{\mathrm{tot}}+\hat{p})\log\left(\frac{n'-\hat{C}_{\mathrm{tot}}+\hat{p}}{n'}\right). 
\end{align*}
Then evaluating $\frac{-ds_n}{d\hat{p}}$, 
\begin{align*}  
&\begin{multlined}[0.8\linewidth]  -\frac{ds_n}{d\hat{p}}=  \log \frac{\hat{p}}{n} + \hat{p}\cdot\frac{n}{\hat{p}}\cdot\frac{1}{n} -  \log \left(\frac{n-\hat{p}}{n}\right) -(n - \hat{p})\cdot \frac{n}{n-\hat{p}}\cdot \frac{1}{n} \\ - \log \frac{\hat{C}_{\mathrm{tot}}-\hat{p}}{n'}- (\hat{C}_{\mathrm{tot}}-\hat{p})\cdot\frac{n'}{\hat{C}_{\mathrm{tot}}-\hat{p}} \cdot \frac{1}{n'} + \log\left(\frac{n'-\hat{C}_{\mathrm{tot}}+\hat{p}}{n'}\right)+ (n'-\hat{C}_{\mathrm{tot}}+\hat{p})\cdot\frac{n'}{n'-\hat{C}_{\mathrm{tot}}+\hat{p}} \cdot \frac{1}{n'}
    \end{multlined} \\
    & \phantom{-\frac{ds_n}{d\hat{p}}}=  \log \frac{\hat{p}}{n} + 1 -  \log \left(\frac{n-\hat{p}}{n}\right)- 1 - \log \frac{\hat{C}_{\mathrm{tot}}-\hat{p}}{n'}- 1 + \log\left(\frac{n'-\hat{C}_{\mathrm{tot}}+\hat{p}}{n'}\right)+1 
    \\
    & \phantom{-\frac{ds_n}{d\hat{p}}}=  \log \frac{\hat{p}}{n-\hat{p}} + \log\left(\frac{n'-\hat{C}_{\mathrm{tot}}+\hat{p}}{\hat{C}_{\mathrm{tot}}-\hat{p}}\right) \\
    & \phantom{-\frac{ds_n}{d\hat{p}}}= \log\left(\frac{\hat{p} \{n'-\hat{C}_{\mathrm{tot}}+\hat{p}\}}{(n-\hat{p}) \{\hat{C}_{\mathrm{tot}}-\hat{p}\}}\right).  
\end{align*}
Then for $\hat{p} < \frac{n \hat{C}_{\mathrm{tot}}}{{\binom{M}{2}}}$, $\frac{ds_n}{d\hat{p}} > 0$
and for $\hat{p} \ge \frac{n \hat{C}_{\mathrm{tot}}}{{\binom{M}{2}}}$, $\frac{ds_n}{d\hat{p}} \le 0$. Since $s_n(\hat{p})$ is an increasing, then decreasing function with a stationary point at $\hat{p} = \frac{n \hat{C}_{\mathrm{tot}}}{{\binom{M}{2}}}$, the infimum would occur at either the largest value or smallest value for $1> \frac{\hat{p}}{n} > 0.5$. 

Now, we can simplify the infimum in~\eqref{eq:alt_stopping_inf} as follows,
\begin{align}
\label{eq:simplify_s}
\nonumber
    \inf_{C' \in \Alt(C_t)} s(C') &=  \inf_{n \in [{\binom{M}{2}}]} \inf_{\hat{p}:|N_p|=n} s_n(\hat{p}) \\
    &= \begin{cases}
    \inf_{n \in [{\binom{M}{2}}]} s_n(\max_{|N_p|=n} \hat{p}), & \text{for } \hat{C}_{\mathrm{tot}} > \frac{{\binom{M}{2}}}{2} \\
    \inf_{n \in [{\binom{M}{2}}]} \min\{ s_n(\max_{|N_p|=n} \hat{p}), s_n(\min_{|N_p|=n, \hat{C}_{\mathrm{tot}}-\hat{p}<0.5n'} \hat{p}) \}, & \text{for } \hat{C}_{\mathrm{tot}} \le \frac{{\binom{M}{2}}}{2} 
    \end{cases}. 
\end{align}
Before we go through each case in detail, we establish a property of the objective function. Since $q = \min \{0.5, \frac{\sum_{(i,j)\in N_q} (\hat{c}_t)_{ij}}{|N_q|}\}$, if $\frac{\sum_{(i,j)\in N_q} (\hat{c}_t)_{ij}}{|N_q|}=\frac{\hat{C}_{\mathrm{tot}}-\hat{p}}{n'} > 0.5$, then
\begin{align}
\label{eq:ineq1}
\nonumber
    -s_n(\hat{p}) &= \hat{p} \log \frac{\hat{p}}{n} + (n - \hat{p}) \log \left(\frac{n-\hat{p}}{n}\right) + (\hat{C}_{\mathrm{tot}}-\hat{p})\log \frac{\hat{C}_{\mathrm{tot}}-\hat{p}}{n'} \\
&\quad + (n'-\hat{C}_{\mathrm{tot}}+\hat{p})\log\left(\frac{n'-\hat{C}_{\mathrm{tot}}+\hat{p}}{n'}\right) \\ 
    \nonumber
    &= \hat{p} \log \frac{\hat{p}}{n} + (n - \hat{p}) \log \left(\frac{n-\hat{p}}{n}\right) + 0.5n'\log 0.5 + (n'-0.5n')\log\left(1-0.5\right) \\
    &= \hat{p} \log \frac{\hat{p}}{n} + (n - \hat{p}) \log \left(\frac{n-\hat{p}}{n}\right) + n'\log 0.5. 
\end{align}
For $\hat{q} =\hat{C}_{\mathrm{tot}}-\hat{p} 
\le 0.5n'$, since $x \log x + (1-x) \log (1-x) > \log 0.5$ for $ x \in [0,1]$, 
\begin{align}
\label{eq:ineq2}
\nonumber
&(\hat{C}_{\mathrm{tot}}-\hat{p})\log \frac{\hat{C}_{\mathrm{tot}}-\hat{p}}{n'} + (n'-\hat{C}_{\mathrm{tot}}+\hat{p})\log\left(\frac{n'-\hat{C}_{\mathrm{tot}}+\hat{p}}{n'}\right) \\
&= n' \frac{\hat{q}}{n'} \log \frac{\hat{q}}{n'} + (n'- n' \frac{\hat{q}}{n'})\log\left(\frac{n'-\hat{q}}{n'}\right) \\
\nonumber
&= n' \frac{\hat{q}}{n'} \log \frac{\hat{q}}{n'} + n'(1- \frac{\hat{q}}{n'})\log\left(1-\frac{\hat{q}}{n'}\right) \\
\nonumber
&= n' \left( \frac{\hat{q}}{n'} \log \frac{\hat{q}}{n'} + (1- \frac{\hat{q}}{n'})\log\left(1-\frac{\hat{q}}{n'}\right) \right) \\
&\ge n' \log 0.5. 
\end{align}
Combining the inequalities~\eqref{eq:ineq1} and~\eqref{eq:ineq2}, for $\hat{p}'$ such that $\frac{\hat{C}_{\mathrm{tot}}-\hat{p}'}{n'} > 0.5$, and $\hat{p}$ such that $\frac{\hat{C}_{\mathrm{tot}}-\hat{p}}{n'} \le 0.5$, the following statement holds true
\begin{align*}
    -s_n(\hat{p}') \le -s_n(\hat{p}) \implies  s_n(\hat{p}') \ge s_n(\hat{p}).  
\end{align*}
Hence, going back to~\eqref{eq:simplify_s}, when we are in the case where $\hat{C}_{\mathrm{tot}} > \frac{{\binom{M}{2}}}{2}$, if $\frac{\hat{p}}{n} \le 0.5$,
\begin{align*}
    \hat{C}_{\mathrm{tot}} > \frac{{\binom{M}{2}}}{2} = 0.5n + 0.5n' \implies   \hat{q} = \hat{C}_{\mathrm{tot}} - \hat{p} > \hat{C}_{\mathrm{tot}} - 0.5n >  0.5n'. 
\end{align*}
Then, $s_n\left(\max_{|N_p|=n} \hat{p}\right) \le s_n(\hat{p}')$, for all $\hat{p}' \le 0.5n$, even if $\hat{p}' \le  \frac{n \hat{C}_{\mathrm{tot}}}{{\binom{M}{2}}}$. \\
On the other hand, in the case where $\hat{C}_{\mathrm{tot}} \le \frac{{\binom{M}{2}}}{2}$, it may be possible that both $\frac{\hat{p}}{n} \le 0.5$ and $\frac{\hat{q}}{n'} \le 0.5$. Therefore, to minimize $s_n(\hat{p})$, we should find the smallest $\hat{p}$ such that $\frac{\hat{C}_{\mathrm{tot}}-\hat{p}}{n'}=\frac{\hat{q}}{n'} \le 0.5$. Replacing $\hat{p}$ with $n \hat{p}_n$ and $N_p$ (resp. $N_q$) with $N'_p$ (resp. $N'_q$) completes the proof.  
\end{proof}

\section{On the Rank of the Active Clustering Problem}
\label{app:sec:rank-of-active-clustering}

In this section, we show that given $M$ items, the rank of the problem of identifying the underlying clustering of the items under the fixed-confidence regime is equal to $M-1$. Before we present the formal arguments, we reproduce here the definition of rank of a pure exploration problem from \cite{kaufmann2021mixture}.

\begin{definition}\cite[Definition~20]{kaufmann2021mixture}
    \label{def:rank-of-pure-exploration-problem}
    Fix constants $d, S, \bar{M}, R \in \mathbb{N}$. A sequential identification problem specified by a partition $\mathcal{O}=\bigcup_{s=1}^{S} \mathcal{O}_s$, where $O_s \subseteq \mathbb{R}^d$ for all $s \in [S]$, is said to have rank $R$ if for every $s \in [S]$,
    \begin{equation}
        \mathcal{O} \setminus \mathcal{O}_s = \bigcup_{v=1}^{\bar{M}} \bigg\{\boldsymbol{\lambda} \in \mathbb{R}^d: (\lambda_{k_1^{s,v}}, \ldots, \lambda_{k_R^{s,v}}) \in \mathcal{L}_{s, v}\bigg\}
        \label{eq:rank-definition}
    \end{equation}
    for a family of indices $k_r^{s,v} \in [d]$ and open sets $\mathcal{L}_{s,v}$ indexed by $r \in [R]$, $s \in [S]$, and $v \in [\bar{M}]$. In other words, the problem has rank $R$ if for each $s$, the set $\mathcal{O} \setminus \mathcal{O}_s$ is a finite union of sets that are each defined in terms of only $R$-tuples.
\end{definition}

In our context, we have
\[
\mathcal{O} = \mathcal{C} =  \bigcup_{K=1}^{M} \mathcal{C}^K = \bigcup_{K=1}^{M} \ \bigcup_{C \in \mathcal{C}^K} [C]. 
\]
Note that equivalence classes of the form $[C]$ partition the space $\mathcal{C}$ of all problem instances and represent all the possible cluster assignments of $M$ items with $K\in [M]$ clusters. 
Hence, the constant $\bar{M}$ is equal to the total number of possible partitions of $M$ items, popularly known as the Bell number \cite{moser1955asymptotic}.

If $C \in \mathcal{C}^K$, then the analogue of \eqref{eq:rank-definition} in our context is 
\[
\mathcal{C} \setminus [C] = \underbrace{\left(\bigcup_{\substack{\ell \in [M]: \\ \ell \neq K}} \ \bigcup_{C' \in \mathcal{C}^\ell} [C']\right)}_{\coloneqq  \ \mathcal{C}_1} \cup \underbrace{\left(\bigcup_{\substack{C' \in \mathcal{C}^K: \\ [C'] \neq [C]}} [C']\right)}_{\coloneqq \  \mathcal{C}_2}.
\]
We then note that each $C' \in \mathcal{C}_1$ is specified by $M-\ell$ ones in its associated $C_{=
}^{'}$ matrix for some $\ell \neq K$, and each $C' \in \mathcal{C}_2$ is specified by $M-K$ ones in its associated $C_{=}^{'}$ matrix. Thus, the rank of the problem is given by
\[
R = \max\{M-\ell: \ell \in [M]\} = M-1.
\]
In contrast, the rank of a best arm identification problem is $2$, even in the setting of multi-objective best arm identification, as noted in \cite{chen25b}. Thus, for large values of $M$, the rank of the crowdsourced clustering problem is exceedingly larger than that of best arm identification.

\bibliographystyle{IEEEtran}
\bibliography{alt-2026}

\end{document}